\documentclass[11pt]{article}

\usepackage[preprint]{acl}

\usepackage{times}
\usepackage{latexsym}
\usepackage[T1]{fontenc}
\usepackage[utf8]{inputenc}
\usepackage{microtype}
\usepackage{inconsolata}
\usepackage{amsmath,amssymb,mathtools,amsthm}
\usepackage{booktabs}
\usepackage{graphicx}
\usepackage{placeins}

\newtheorem{proposition}{Proposition}

\newcommand{\R}{\mathbb{R}}
\newcommand{\E}{\mathbb{E}}
\newcommand{\KL}{D_{\mathrm{KL}}}
\newcommand{\CE}{\mathcal{L}_{\mathrm{CE}}}
\newcommand{\softmax}{\operatorname{softmax}}
\newcommand{\sigmoid}{\sigma}
\newcommand{\diag}{\operatorname{diag}}

\newcommand{\TwoNN}{\operatorname{TwoNN}}
\newcommand{\1}{\mathbf{1}}
\newcommand{\calC}{\mathcal{C}}
\newcommand{\calD}{\mathcal{D}}
\newcommand{\calE}{\mathcal{E}}
\newcommand{\calL}{\mathcal{L}}
\newcommand{\calM}{\mathcal{M}}
\newcommand{\calT}{\mathcal{T}}
\newcommand{\calZ}{\mathcal{Z}}
\newcommand{\tr}{\mathrm{tr}}
\newcommand{\cali}{\mathrm{cal}}
\newcommand{\test}{\mathrm{test}}
\newcommand{\calJ}{\mathcal{J}}

\title{When Do Concepts Become Functionally Sufficient During Language-Model Training?}

\author{
\textbf{Raphael Bernas}\textsuperscript{1} \quad
\textbf{Paul G. Chevalier}\textsuperscript{1} \quad
\textbf{Fanny Jourdan}\textsuperscript{2,3} \quad
\textbf{Céline Hudelot}\textsuperscript{1} \\
\\
\textsuperscript{1}MICS, CentraleSupelec, Université Paris-Saclay \\
\textsuperscript{2}IRT Saint Exupéry, Toulouse, France \\
\textsuperscript{3}Mila - Quebec AI Institute,
Montreal, Canada\\
\\
\texttt{raphael.bernas@centralesupelec.fr}
}

\begin{document}
\maketitle
\begin{abstract}
Understanding a model and its learning mechanisms in depth requires identifying when its internal structures become useful, rather than simply looking at the final state. We study this through concept dynamics: at each layer and checkpoint, we decompose activations, select sparse soft masks, and inject masked reconstructions into the model. Concept analysis is therefore tested functionally: a mask is useful only insofar as it preserves a target under intervention. We compare sufficiency for activation reconstruction, linear decodability, true downstream preservation, and checkpoint transfer under learned alignment. The framework treats decomposition assumptions as hypotheses rather than interpretability guarantees, monitoring functional sufficiency across checkpoints and source-to-final reconstructability under learned alignment. At the shared fixed-penalty operating point across seven models, downstream masks retain substantially less soft mass than reconstruction masks; predictive-distribution shifts remain small.
\end{abstract}

\section{Introduction}

Modern neural models represent task information in high-dimensional activations. Understanding their learning requires asking what appears, how it is organized, and when it becomes useful for prediction. This motivates measuring how internal structures acquire functional status during training.

Concept-based methods are a natural step, but decorrelation, positivity, independence, or sparsity define coordinate systems rather than semantic guarantees. Their functional relevance should therefore be tested by intervention, not reconstruction alone.

We study this question through \emph{concept dynamics}. At each layer and checkpoint, a sparse mask selects coordinates from a fitted activation decomposition; the masked reconstruction is then evaluated in the original model on held-out data. Because represented information need not be used \citep{hewitt-liang-2019-designing,elazar-etal-2021-amnesic}, we compare reconstruction, fixed-head decodability, true downstream preservation, and checkpoint transfer.

The unit of analysis is a time-indexed sufficiency claim: at checkpoint $t$, a sparse soft mask preserves a specified target at a given layer. Relevant structures may be human-interpretable, distributed, or unnamed.

Our contributions are:

\begin{itemize}
    \item We introduce a checkpoint-wise activation-intervention framework that compares four preservation
    targets: reconstruction, fixed-head decodability, native-downstream behavior, and learned-alignment transfer.

    \item We relate native-downstream KL preservation to local predictive geometry and derive an exact
    interpretation of the supervised component-gate moments recorded by the experiments.

    \item Across seven checkpointed language models, we report target-conditional sparse operating points and
    model-specific checkpoint and depth profiles.
\end{itemize}

\section{Related Work}
\label{sec:related_work}

\paragraph{Mechanistic studies of training dynamics.}
Training dynamics have been studied through representation similarity (SVCCA and CKA) \citep{raghu2017svcca,kornblith2019similarity}, checkpointed suites \citep{biderman2023pythia}, activation convergence \citep{diehl-martinez-etal-2024-tending}, mechanistic progress measures \citep{nanda2023progress}, and feature emergence \citep{xu2024trackingdynamics,inaba-etal-2025-bilingual}. Geometric work tracks anisotropy, intrinsic dimension, and direction amplification \citep{razzhigaev2024shapelearninganisotropyintrinsic,bernas2026revisitinganisotropy}. It studies training geometry; we instead test checkpoint-wise held-out interventions with multiple preservation targets and learned transfer.

\paragraph{Concepts inside models.}
Concept activation vectors use human-specified examples \citep{kim2018interpretability}; unsupervised approaches decompose activations with SVD/PCA \citep{graziani2023uncovering}, NMF \citep{collins2018deep,fel2023craft,jourdan-etal-2023-cockatiel,zhang2021invertible}, or sparse autoencoders \citep{cunningham2023sparse,bricken2023monosemanticity}, with unifying comparisons in \citet{fel2023holistic,poche-etal-2025-consim}. Output-preserving and gradient-aware dictionaries connect coordinates to function \citep{braun2024functionally,olmo2025features}; supervised PCA likewise distinguishes high-variance from predictive directions \citep{bair2006prediction}. Probing and causal abstraction separate represented from used information \citep{hewitt-liang-2019-designing,elazar-etal-2021-amnesic,geiger2025causal}. We follow this intervention view while comparing fixed decompositions and targets across layers and checkpoints.

\section{Concept-Mask Interventions}
\label{sec:concept_subset_interventions}

To make sufficiency a held-out claim, we separate representation fitting, mask calibration, and intervention evaluation. Within a model, let $H_{t,s}^{(\ell)}\in\R^{n_s\times d_\ell}$ denote the activations at layer $\ell$, checkpoint $t$, and split $s\in\{\tr,\cali,\test\}$. We use $\calD_{\tr}$ to fit representations and auxiliary maps, $\calD_{\cali}$ to select masks, and $\calD_{\test}$ only for final measurements. The candidate budget $K_{\max}$ is model-specific and shared across its checkpoints, layers, and methods.

For each extraction method $m\in\calM$, we fit a dictionary, affine reconstruction baseline, and code-inference map:
\begin{align}
V_t^{(\ell,m)}&\in\R^{d_\ell\times K_{\max}},
&b_t^{(\ell,m)}&\in\R^{d_\ell},
\nonumber\\
\calC_t^{(\ell,m)}&:\R^{n\times d_\ell}\to\R^{n\times K_{\max}}.
\end{align}
The code matrix and full reconstruction on split $s$ are
\begin{align}
U_{t,s}^{(\ell,m)}
&=
\calC_t^{(\ell,m)}(H_{t,s}^{(\ell)}),
\label{eq:codes}\\
\widetilde H_{t,s}^{(\ell,m)}
&=
\1\bigl(b_t^{(\ell,m)}\bigr)^\top
+U_{t,s}^{(\ell,m)}\bigl(V_t^{(\ell,m)}\bigr)^\top.
\label{eq:full_reconstruction}
\end{align}
Here $b$ is zero for SVD and SemiNMF, the fitted mean for ICA, and the negative fitted shift for NMF. The notation covers sparse dictionaries \citep{cunningham2023sparse,bricken2023monosemanticity} and matrix factorizations \citep{jolliffe2002principal,lee1999learning,ding2010convex,hyvarinen2000independent}. Their constraints propose coordinates; intervention determines relevance to a target.

The maximum candidate budget is fixed from training activations only. In the executed protocol, the two-nearest-neighbor estimate \citep{facco2017estimating} is first clipped to a conservative interval,
\begin{align}
\widehat d_{t,\ell}
&=
\min\!\left\{d_\ell,
\max\!\left\{\frac{d_\ell}{4},
\TwoNN(H_{t,\tr}^{(\ell)})\right\}\right\},
\nonumber\\
\bar d
&=
\frac{1}{|L|\,|\calT|}
\sum_{\ell\in L}\sum_{t\in\calT}
\widehat d_{t,\ell},
\nonumber\\
K_{\max}
&=
\left\lceil 2\bar d+1\right\rceil .
\label{eq:kmax}
\end{align}
The factor of two expands the clipped estimate. Because the lower clip is active throughout (Appendix~\ref{app:experimental_conditions}), executed $K_{\max}$ is a conservative width-scaled capacity; mask calibration determines retention. A selection is represented by a soft mask
\begin{equation}
z(\alpha;\vartheta)_i=\sigmoid(\alpha_i/\vartheta),
\qquad
z(\alpha;\vartheta)\in[0,1]^{K_{\max}},
\label{eq:soft_mask}
\end{equation}
where $\vartheta>0$ is a gate temperature and $\alpha\in\R^{K_{\max}}$ is optimized on $\calD_{\cali}$. For fixed codes, dictionary, and baseline, the masked activation is
\begin{equation}
\widehat H_{t,s,z}^{(\ell,m)}
=
\1\bigl(b_t^{(\ell,m)}\bigr)^\top
+U_{t,s}^{(\ell,m)}
\diag(z)
\bigl(V_t^{(\ell,m)}\bigr)^\top.
\label{eq:masked_activation}
\end{equation}
Each mask is selected by minimizing a preservation loss with an explicit sparsity penalty:
\begin{equation}
F_c(\alpha;\vartheta)
=
\calL_c\bigl(z(\alpha)\bigr)
+
\lambda\|z(\alpha)\|_1,
\qquad
\lambda\ge0.
\label{eq:generic_mask_objective}
\end{equation}
Inside objectives, $z(\alpha)$ abbreviates the current-temperature mask. We run a finite number of Adam updates \citep{kingma2015adam} while geometrically annealing $\vartheta$ from $1.0$ to $0.1$, and return $\widehat\alpha_c$, with $z_c=z(\widehat\alpha_c;0.1)$; exact minimization is not assumed. Test evaluation replaces $H_{t,\test}^{(\ell)}$ by $\widehat H_{t,\test,z_c}^{(\ell,m)}$ in the original graph. We use the prediction-first convention $\CE(q,r)=-\E\sum_y r(y)\log q(y)$, averaged over rows. The reporting quantities are
\begin{align}
\Delta\mathrm{CE}^{y}_c
&=
\CE(\widehat p_{c,\test},y_{\test})
-
\CE(p_{\test},y_{\test}),
\label{eq:delta_ce}\\
\Delta\mathrm{KL}_c
&=
\KL\!\left(p_{\test}\;\middle\|\;\widehat p_{c,\test}\right),
\label{eq:delta_kl}\\
\Delta\mathrm{Acc}_c
&=
\mathrm{Acc}(\widehat p_{c,\test},y_{\test})
-
\mathrm{Acc}(p_{\test},y_{\test}).
\label{eq:delta_acc}
\end{align}
Writing $\bar h_{t,\test}^{(\ell)}$ for the mean test activation row, the activation error is
\begin{equation}
\mathrm{NMSE}_c
=
\frac{\|H_{t,\test}^{(\ell)}-
\widehat H_{t,\test,z_c}^{(\ell,m)}\|_F^2}
{\|H_{t,\test}^{(\ell)}-
\1(\bar h_{t,\test}^{(\ell)})^\top\|_F^2}.
\label{eq:intervention_nmse}
\end{equation}
The denominator is floored at $10^{-12}$. Kullback-Leibler divergence (KL) is evaluated as the soft-target cross-entropy difference using the original predictions as targets, whereas $\Delta\mathrm{CE}^{y}_c$ uses gold labels. We report mask size as $K_{\mathrm{soft}}(z)=\sum_i z_i$ and $K_{\mathrm{hard}}(z)=\sum_i\1\{z_i>1/2\}$.

\section{Preservation Objectives}
\label{sec:preservation_objectives}

Holding the candidate decomposition fixed, we vary the selection target: geometry, fixed-head access, true downstream behavior, or cross-checkpoint reconstructability. This progression separates information that is easy to reconstruct from information accessible to an adapter and used by the model itself.

\paragraph{Activation reconstruction.}
The geometric baseline retains activation variation by minimizing variance-normalized Euclidean error. Let $\bar h_{t,\cali}^{(\ell)}$ be the mean calibration activation row and
\begin{align}
\bar s_{t,\cali,\ell}^2
&=
\frac{1}{n_{\cali}d_\ell}
\left\|H_{t,\cali}^{(\ell)}-
\1(\bar h_{t,\cali}^{(\ell)})^\top\right\|_F^2,
\nonumber\\
s_{t,\cali,\ell}^2
&=\max\{\bar s_{t,\cali,\ell}^2,10^{-12}\}.
\end{align}
The executed objective is
\begin{equation}
\calL_{\mathrm{rec}}^{(\ell,t,m)}(z)
=
\frac{1}{n_{\cali}d_\ell s_{t,\cali,\ell}^2}
\left\|
H_{t,\cali}^{(\ell)}
-
\widehat H_{t,\cali,z}^{(\ell,m)}
\right\|_F^2.
\label{eq:rec_loss}
\end{equation}
The normalization permits layer comparison without changing within-cell error orderings; geometric fidelity alone need not preserve downstream computation.

\paragraph{Fixed-head decodability.}
The second objective preserves information accessible through an affine translator $\calE_\theta:\R^{d_\ell}\to\R^{d_{\mathrm{head}}}$ into the checkpoint-$t$ native head $g_t$. Define $p_\theta^{\mathrm{lin}}(H)=\softmax(g_t(\calE_\theta(H)))$. AdamW \citep{loshchilov2019decoupled} trains $\theta$ on $\calD_{\tr}$ to reduce
\begin{equation}
\calJ_{\mathrm{lin}}(\theta)
=
\CE\!\left(
p_\theta^{\mathrm{lin}}(H_{t,\tr}^{(\ell)}),
y_{\tr}
\right).
\label{eq:linear_training}
\end{equation}
The mask then preserves the induced distribution:
\begin{equation}
\calL_{\mathrm{lin}}^{(\ell,t,m)}(z)
=
\KL\!\left(
p_{\widehat\theta}^{\mathrm{lin}}(H_{t,\cali}^{(\ell)})
\;\middle\|\;
p_{\widehat\theta}^{\mathrm{lin}}(\widehat H_{t,\cali,z}^{(\ell,m)})
\right).
\label{eq:linear_loss}
\end{equation}
Here $\widehat\theta$ denotes the returned parameters. This resembles linear probing \citep{hewitt-liang-2019-designing}: for autoregressive models the fixed head is linear, while for masked models the affine translator feeds the fixed nonlinear MLM head. ``Linear'' is therefore table shorthand for fixed-head decodability; final evaluation still intervenes in the full model.

\paragraph{True downstream preservation.}
Replacing the adapter with the actual remainder of the network turns accessibility into a direct behavioral criterion. Let $f_{t,>\ell}$ be the computation after layer $\ell$ at checkpoint $t$, and define
\begin{equation}
p_t^{(\ell)}(H)
=
\softmax\!\left(g_t(f_{t,>\ell}(H))\right).
\label{eq:true_downstream_prob}
\end{equation}
The preservation loss is
\begin{align}
\calL_{\mathrm{true}}^{(\ell,t,m)}(z)
=
\KL\!\left(
p_t^{(\ell)}(H_{t,\cali}^{(\ell)})
\;\middle\|\;
\right.
\nonumber\\
\left.
p_t^{(\ell)}(\widehat H_{t,\cali,z}^{(\ell,m)})
\right).
\label{eq:true_loss}
\end{align}
This directly tests whether masked reconstruction preserves the model's prediction distribution, grounding representation-use claims in an internal counterfactual \citep{elazar-etal-2021-amnesic,geiger2025causal}.

\paragraph{Checkpoint transfer.}
Finally, transfer asks which checkpoint-$t$ concepts reconstruct checkpoint $T$. An affine $A_t:\R^{K_{\max}}\to\R^{K_{\max}}$ is trained on $\calD_{\tr}$ by aligning row-wise softmax-normalized code profiles. Write
\begin{align}
\Pi_t(A_t)
&=\softmax(A_t(U_{t,\tr}^{(\ell,m)})),
\nonumber\\
\Pi_T
&=\softmax(U_{T,\tr}^{(\ell,m)}).
\label{eq:alignment_distributions}
\end{align}
The same AdamW optimizer, with decoupled weight-decay coefficient $\rho$, trains $A_t$ to reduce
\begin{align}
\calJ_{\mathrm{align}}(A_t)
=
\KL\!\left(\Pi_t(A_t)\;\middle\|\;\Pi_T\right).
\label{eq:alignment_training}
\end{align}
The returned map $\widehat A_t$ decodes masked source codes with the checkpoint-$T$ dictionary and baseline:
\begin{equation}
\begin{aligned}
\widehat H_{T\leftarrow t,s,z}^{(\ell,m)}
&=\1\bigl(b_T^{(\ell,m)}\bigr)^\top\\
&\quad+\widehat A_t\!\left(U_{t,s}^{(\ell,m)}\diag(z)\right)
\bigl(V_T^{(\ell,m)}\bigr)^\top.
\end{aligned}
\label{eq:transfer_reconstruction}
\end{equation}
The calibration loss combines later-activation reconstruction and normalized-code alignment. Let
\begin{align}
\Pi_{T,\cali}
&=\softmax(U_{T,\cali}^{(\ell,m)}),
\nonumber\\
\widehat\Pi_{z,\cali}
&=
\softmax\!\Bigl(
\widehat A_t(U_{t,\cali}^{(\ell,m)}\diag(z))
\Bigr).
\label{eq:calibration_code_dists}
\end{align}
Abbreviate the calibration reconstruction in Equation~\eqref{eq:transfer_reconstruction} as $\widehat H_z^{T\leftarrow t}$. With $s_{T,\cali,\ell}^2$ defined as in the reconstruction objective but at checkpoint $T$, the transfer loss is
\begin{align}
\calL_{\mathrm{dist}}^{(\ell,t,m)}(z)
&=
\frac{1}{n_{\cali}d_\ell s_{T,\cali,\ell}^2}
\left\|
H_{T,\cali}^{(\ell)}
-
\widehat H_z^{T\leftarrow t}
\right\|_F^2
\nonumber\\
&\quad+
\mu\,
\KL\!\left(\Pi_{T,\cali}\|\widehat\Pi_{z,\cali}\right).
\label{eq:dist_loss}
\end{align}
The asymmetric directions match execution: alignment uses mapped-source-to-target KL, whereas mask fitting uses target-to-masked-aligned KL. Unlike representation similarity \citep{raghu2017svcca,kornblith2019similarity}, transfer asks whether soft-masked earlier coordinates reconstruct a later activation that preserves held-out behavior.

All four objectives use the same held-out intervention mechanism; their disagreement therefore reveals target-specific sufficiency. The next section explains two sources of it.

\section{Relations Among the Monitors}
\label{sec:theory}

Relating the four operational definitions explains why the same candidates yield different sets. Downstream KL is locally curvature-weighted reconstruction error, and the saved supervised-gradient scores resolve this sensitivity coordinate by coordinate. Fix $t$, $\ell$, and $m$ and omit these indices. For example $X$, let
\begin{align}
\widehat a_z(X)
&=b+\sum_{i=1}^{K_{\max}}z_i u_i(X)v_i,
\nonumber\\
\Delta_z(X)&=\widehat a_z(X)-a(X).
\end{align}
Let $r_t^{(\ell)}(a)$ denote the downstream logit map and $p_t(\cdot\mid a)=\softmax(r_t^{(\ell)}(a))$. The population form of the downstream objective is
\begin{equation}
D_t(z)=
\E_X\KL\!\left(
p_t(\cdot\mid a(X))
\middle\|
p_t(\cdot\mid \widehat a_z(X))
\right).
\label{eq:theory_downstream}
\end{equation}

\paragraph{Local downstream geometry.}
Reconstruction treats every residual direction equally, whereas the downstream network can amplify some directions and ignore others. Write $J_t(a)=\partial r_t^{(\ell)}(a)/\partial a$ and $S(p)=\diag(p)-pp^\top$. The predictive Fisher metric pulled back to the intervention layer makes this weighting explicit:
\begin{align}
G_t(a)
&=J_t(a)^\top S(p_t(\cdot\mid a))J_t(a),
\nonumber\\
Q_t(z)
&=\frac12\E_X\!\left[
\Delta_z^\top G_t(a)\Delta_z
\right].
\label{eq:local_downstream_quad}
\end{align}
The second-order KL--Fisher expansion is standard in information geometry \citep{amari1998natural,martens2020new} and has recently been stated for activation steering \citep{wang2026fishback}. The following proposition records the remainder needed when masks are compared.

\begin{proposition}[Local comparison of preservation losses]
\label{prop:local_downstream_approx}
Let $\calZ$ be the compared mask class, let $P:\calZ\to\R$ be a shared penalty, and write $F_D=D_t+P$ and $F_Q=Q_t+P$. Suppose that, for every $z\in\calZ$ and almost every $X$, the downstream KL is $C^3$ along the segment from zero to $\Delta_z(X)$, with third-derivative operator norm at most $B_X(z)$ and $\E[B_X(z)\|\Delta_z(X)\|_2^3]<\infty$. Then
\begin{align}
\eta_t(z)&:=\frac16\E_X\!\left[B_X(z)\|\Delta_z(X)\|_2^3\right],
\nonumber\\
|D_t(z)-Q_t(z)|&\le\eta_t(z).
\label{eq:local_downstream_bound}
\end{align}
If $z_D$ and $z_Q$ minimize $F_D$ and $F_Q$, respectively, over $\calZ$, then
\begin{equation}
F_D(z_Q)-F_D(z_D)
\le
\eta_t(z_D)+\eta_t(z_Q).
\label{eq:local_optimizer_comparison}
\end{equation}
\end{proposition}

With $P(z)=\lambda\|z\|_1$, this matches mask selection: within the controlled neighborhood, optimizing the quadratic incurs only a third-order regularized loss gap.

Let $R_t(z)=\E_X\|\Delta_z(X)\|_2^2$, the unnormalized counterpart of the reconstruction loss within a fixed experimental cell. For $R_t(z)>0$, define the curvature in the residual directions by
\[
\bar\lambda_t(z)=
\frac{\E_X[\Delta_z(X)^\top G_t(a(X))\Delta_z(X)]}
     {\E_X\|\Delta_z(X)\|_2^2}.
\]
Then the exact factorization
\begin{equation}
Q_t(z)=\frac12\bar\lambda_t(z)R_t(z)
\label{eq:reconstruction_anisotropy}
\end{equation}
shows where the objectives separate. A mask may leave a large residual in insensitive directions or a small residual in high-curvature directions, reversing the Euclidean ordering within the local approximation. This mechanism is independent of the decomposition.

The same local quantity can be written as a directional sensitivity. For $Y\mid X\sim p_t(\cdot\mid a(X))$, let $\xi_t(a,Y)=\nabla_a\log p_t(Y\mid a)$. The standard Fisher-score identity yields
\begin{equation}
Q_t(z)
=
\frac12\E_{X,Y}
\left[
\left\langle\Delta_z(X),\xi_t(a(X),Y)\right\rangle^2
\right].
\label{eq:gradient_fisher_objective}
\end{equation}
The experiments use gold-label rather than model-sampled gradients. These supervised second moments are empirical-Fisher-like quantities and need not equal Equation~\eqref{eq:gradient_fisher_objective} \citep{kunstner2019limitations}.

\paragraph{Supervised concept-gate scores.}
Equation~\eqref{eq:gradient_fisher_objective} scores a complete intervention residual. To obtain a coordinate-resolved diagnostic, let $c_i(X)=u_i(X)v_i$ be the contribution of concept $i$ and introduce gates around the unmodified activation,
\begin{equation}
a_s(X)=a(X)+\sum_i(s_i-1)c_i(X).
\label{eq:diagnostic_gate}
\end{equation}
This definition satisfies $a_{\mathbf 1}(X)=a(X)$ even when the decomposition has a reconstruction residual. For the gold-label loss $\ell$, define
\begin{equation}
g_i(X)=
\left.\frac{\partial\ell(X;s)}{\partial s_i}\right|_{s=\mathbf 1}
=\langle c_i(X),\nabla_a\ell(X)\rangle .
\label{eq:gate_gradient}
\end{equation}
The stored scores are
\begin{align}
\tau_i&=\E g_i^2,
&\omega_i&=(\E g_i)^2,
\nonumber\\
\kappa_i&=\tau_i-\omega_i
=\mathrm{Var}(g_i).
\label{eq:gate_moments}
\end{align}
Here $\E$ denotes the empirical average over the scored training rows.
They are related to gradient-based feature sensitivity and first-order intervention approximations \citep{liu2021group,syed2024attribution}, but remain local gate statistics.

\begin{proposition}[Batch averaging of concept-gate gradients]
\label{prop:gate_averaging}
Let $X_1,\ldots,X_B$ be IID draws with replacement from the scored-row distribution and, for a concept set $A$, let
$\bar g_{B,A}=B^{-1}\sum_{b=1}^B(g_i(X_b))_{i\in A}$. If the gate gradients have finite second moments, then
\begin{equation}
\E\|\bar g_{B,A}\|_2^2
=\sum_{i\in A}\left[
\left(1-\frac1B\right)\omega_i+\frac{\tau_i}{B}
\right].
\label{eq:gate_batch_identity}
\end{equation}
\end{proposition}

Thus $\tau$ measures total single-example sensitivity, whereas $\omega$ isolates the signed component preserved by averaging. Equation~\eqref{eq:gate_batch_identity} is the variance-of-the-mean identity for concept gates \citep{mccandlish2018empirical,chatterjee2020coherence}, independent of orthogonality. These are local diagnostics, not realized updates or occupied dimensions. Appendix~\ref{app:theory_proofs} gives the proof and SVD-prefix consequence.

\section{Empirical Results}
\label{sec:empirical}

We study three Pythia models \citep{biderman2023pythia}, OLMo-1B \citep{groeneveld-etal-2024-olmo,allenai2024olmo1b0724}, EuroBERT-210M/610M \citep{boizard2025eurobert}, and ModernCamemBERT \citep{antoun-etal-2025-modernbert}. The paired common grid has eight checkpoints, six layer ranks, four decompositions, and four targets (Appendix~\ref{app:experimental_conditions}). We ask whether target-aware masks are compact and have structured checkpoint/depth profiles, and how decomposition affects them.

Dictionaries and auxiliary maps use up to 8,000 eligible training rows, masks use a disjoint 4,000-sequence calibration split, and reported effects use 3,000 held-out sequences, retaining one eligible target-token row per sequence. SVD, ICA, NMF, and SemiNMF share the within-model candidate budget. Cross-objective contrasts therefore match model, method, checkpoint, and layer exactly; the fixed shared penalty defines one operating point rather than an estimated sparsity--fidelity frontier.

Behavioral results use the fitted soft gates, with $K_{\rm soft}/K_{\max}=\sum_i z_i/K_{\max}$; $K_{\rm hard}/K_{\max}=|\{i:z_i>0.5\}|/K_{\max}$ summarizes their thresholded composition. Relative KL and CE use original target-token CE as denominator, and accuracy damage is $100(\mathrm{Acc}_{0}-\mathrm{Acc}_{z})$. Each dashboard entry first takes the median over the balanced $4\times8\times6$ method--checkpoint--layer grid within model and then the median [IQR] over seven models; empty rates are model macro-averages.

\subsection{Target-aware selection finds compact operating points}

Figure~\ref{fig:functional_synthesis} provides the common view: panels A--C trace held-out behavioral effects, panel D directly pairs downstream and reconstruction selection, and panel E reports source-to-final transfer. The transfer endpoint is learned self-alignment and is therefore presented separately from the three same-checkpoint objectives.

\begin{figure*}[t]
\centering
\includegraphics[width=.92\textwidth]{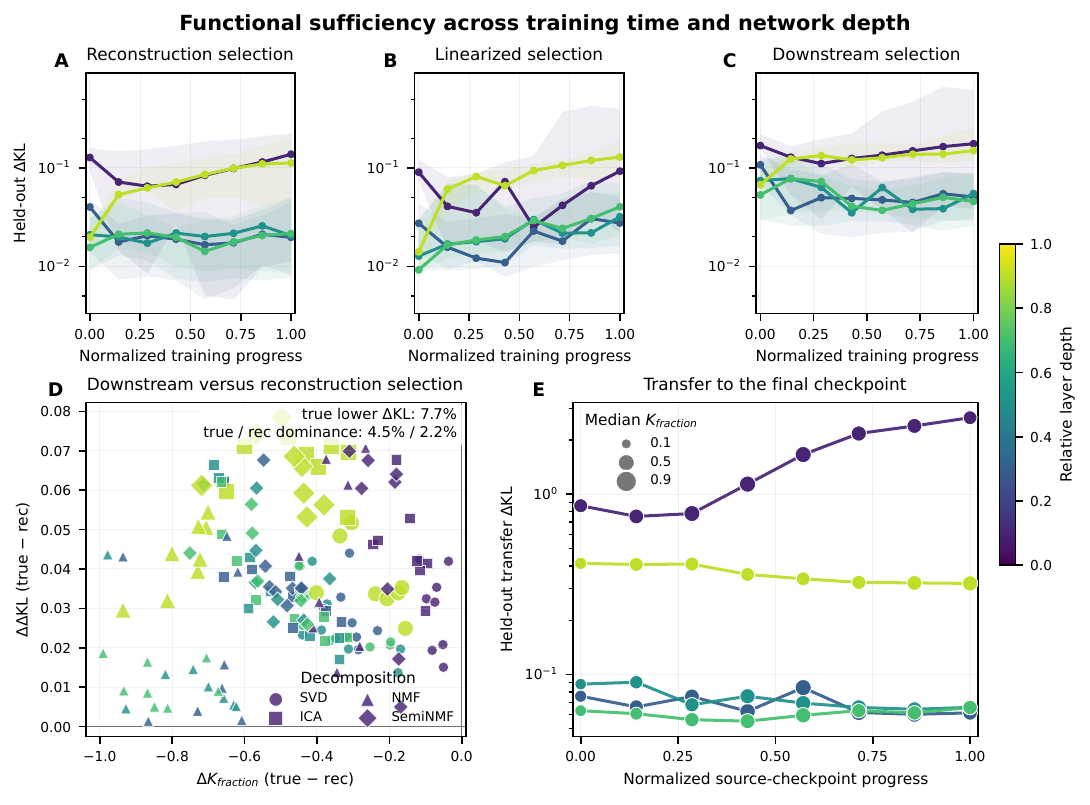}
\caption{\textbf{Functional sufficiency across checkpoints and layers.} A--C: held-out $\Delta\mathrm{KL}$ for reconstruction, fixed-head (``linearized'' in artwork), and downstream selection, using ordinal checkpoint and layer ranks. D: paired downstream-minus-reconstruction occupancy and $\Delta\mathrm{KL}$ changes. E: learned-alignment source-to-final reconstructability. Curves and bands are hierarchical medians and IQRs over seven model summaries.}
\label{fig:functional_synthesis}
\end{figure*}

\begin{table}[t]
\centering
\small
\setlength{\tabcolsep}{2.6pt}
\caption{Seven-model operating points.}
\label{tab:main-operating-points}
\begin{tabular}{@{}lrrrr@{}}
\toprule
Objective & Soft (\%) & Rel. KL & Acc. $\downarrow$ & NMSE \\
\midrule
Reconstruction & 73.0 & 0.009 & 0.37 & 0.215 \\
Fixed-head & 83.4 & 0.014 & 0.47 & 0.270 \\
Downstream & 6.6 & 0.020 & 0.80 & 0.959 \\
\midrule
Transfer & 28.6 & 0.093 & 2.68 & 1.006 \\
\bottomrule
\end{tabular}
\vspace{2pt}

\begin{minipage}{0.98\linewidth}
\footnotesize\textit{Notes.} Medians of seven within-model medians. Soft is $100K_{\rm soft}/K_{\max}$; accuracy damage is in percentage points. Transfer is source-to-final.
\end{minipage}
\end{table}

All three same-checkpoint objectives preserve the predictive distribution closely at their fitted operating points (Table~\ref{tab:main-operating-points}). Reconstruction and fixed-head selection retain $73.0\%$ and $83.4\%$ soft occupancy at relative KL $0.009$ and $0.014$. Downstream selection is markedly more compact: it uses $6.6\%$---about one fifteenth of the candidate budget---at relative KL $0.020$ and $0.80$ percentage-point accuracy damage. Exact pairing places $93.2\%$ of reconstruction--downstream cells in a sparsity--fidelity trade-off class (Table~\ref{tab:pairwise-objectives}), indicating that the target changes the operating point. Transfer is listed separately.

Compactness is not a pooled artifact: every model's downstream median occupancy is below reconstruction, while relative KL ranges from $0.005$ to $0.051$ and target-token accuracy damage from $0.23$ to $2.10$ points (Table~\ref{tab:model-objective-ledger}). More revealingly, downstream masks have median reconstruction NMSE $0.959$, versus $0.215$ for reconstruction selection, while preserving behavior at relative KL $0.020$. This geometry--function separation motivates the local-curvature interpretation below. Reconstruction and fixed-head fidelity are nearly indistinguishable under exact pairing: across 1,344 cells, the median log-KL ratio and relative-CE contrast are both effectively zero, although fixed-head masks retain $8.1$ percentage points more soft occupancy (Table~\ref{tab:pairwise-objectives}). Thus a translator--head path preserves much the same signal without retaining the same soft mass.

\subsection{Training and depth form model-specific profiles}

The pooled curves hide a clearer longitudinal result: each model exhibits a structured combination of checkpoint evolution and depth dependence. Figure~\ref{fig:architecture_fingerprints} summarizes these two axes. The left panel correlates the eight scheduled checkpoint ranks with $\log_{10}\Delta\mathrm{KL}$ within a fixed layer; the right correlates the six tracked-layer ranks with the same metric within a fixed checkpoint. Positive values indicate that the selected-mask effect is larger at later scheduled checkpoints or sampled layers, respectively.

\begin{figure*}[t]
\centering
\includegraphics[width=.84\textwidth]{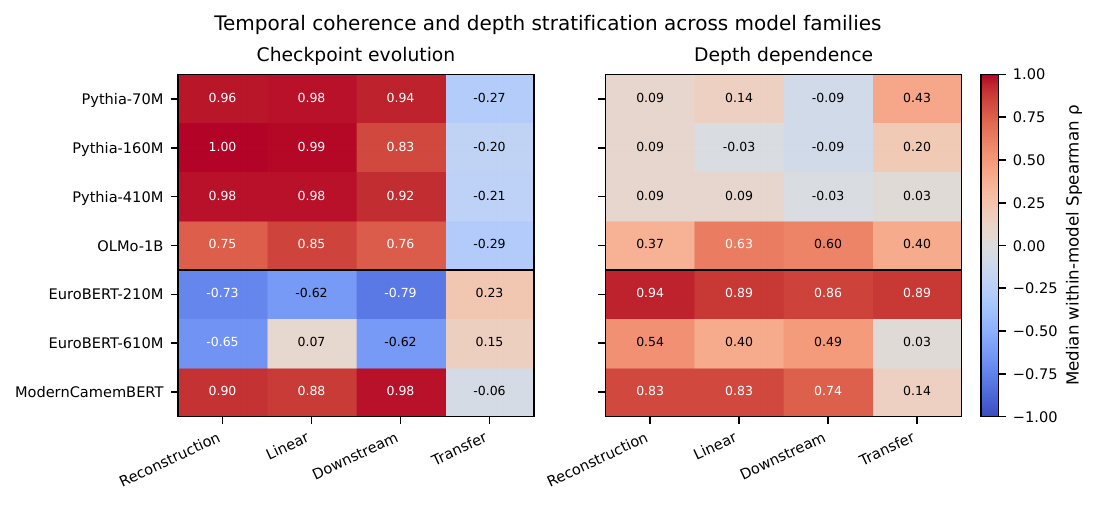}
\caption{\textbf{Model-specific temporal and depth profiles.} Cells are within-model median Spearman correlations over methods and the complementary axis. The horizontal rule separates autoregressive and masked models; differences remain descriptive because their data, tokenizers, objectives, and recipes also differ.}
\label{fig:architecture_fingerprints}
\end{figure*}

For all three same-checkpoint targets, the Pythia models have strong positive checkpoint correlations ($0.83$--$1.00$) and nearly flat layer-rank correlations ($-0.09$--$0.14$). OLMo is likewise positive over checkpoints ($0.75$--$0.85$) but more depth-stratified ($0.37$--$0.63$). ModernCamemBERT is positive on both axes. EuroBERT-210M supplies a coherent contrasting profile: its checkpoint associations are negative ($-0.79$ to $-0.62$), while its layer-rank associations are strong ($0.86$--$0.94$). The agreement across three distinct selection targets within each of these models is itself informative: the monitor recovers a model-level temporal/depth signature rather than a pattern tied to a single loss. EuroBERT-610M and transfer show more mixed profiles, which the layer-resolved trajectories in Appendix~\ref{app:empirical_diagnostics} make explicit.

The within-family size comparison is also structured: Pythia-410M has lower KL than Pythia-70M in $92.8\%$ of matched cells, a smaller normalized mask in $78.8\%$, and both in $72.4\%$ (Figure~\ref{fig:architecture_scale}).

The depth allocation also changes over training. The $\Delta\mathrm{KL}$-weighted center moves toward later sampled layers for reconstruction in six models (median $+0.34$) and fixed-head selection in all seven (median $+0.25$). This indicates a relative relocation of the intervention effect, while downstream selection retains a more model-dependent depth profile (Figures~\ref{fig:depth_relocation} and~\ref{fig:path_shape}). Appendix~\ref{app:empirical_diagnostics} reports the full trajectories and within-family size contrasts.

\subsection{Learned-alignment transfer}

The decomposition provides another informative axis. ICA often reaches lower raw KL with larger masks, whereas SVD has the highest paired within-cell Pareto-dominance rates (Table~\ref{tab:method-behavior}; Figure~\ref{fig:method_pareto}): $0.66$, $0.65$, $0.45$, and $0.33$ for reconstruction, fixed-head, downstream, and transfer, respectively. Thus raw fidelity and sparsity-aware comparison reward different properties of the candidate coordinates. These are protocol-specific operating-point results, not universal method rankings.

The transfer endpoints show a particularly consistent decomposition split. From the earliest source to learned self-alignment, held-out transfer KL is lower for SVD and ICA in all seven models, but higher for NMF in five and SemiNMF in six (Figure~\ref{fig:transfer_methods}). This repeated cross-model contrast explains the nearly flat aggregate curve and indicates that the transfer monitor distinguishes how alignment interacts with each coordinate family. Because the last point is self-alignment, the contrast concerns endpoint behavior rather than monotonic persistence of individual concepts.

\subsection{Downstream masks concentrate supervised sensitivity}

Table~\ref{tab:selected-set-ledger} provides the clearest evidence that the compact downstream masks are compositionally meaningful. In every model, their nonempty hard selections are more enriched for both $\tau$ and $\omega$ than reconstruction selections relative to the uniform same-cardinality expectation. Enrichment divides retained score mass by hard occupancy, so $1\times$ is the analytic size-only baseline. Across-model medians of the independently aggregated model summaries are $7.3\%$ hard occupancy, $49\%$ of total $\tau$ mass, and $75\%$ of total $\omega$ mass. Median enrichments are $3.26\times$ and $5.32\times$, compared with $1.08\times$ for both scores under reconstruction selection; downstream model values span $1.26$--$13.70\times$ for $\tau$ and $1.47$--$16.20\times$ for $\omega$. The larger $\omega$ enrichment adds a mean-alignment interpretation: Proposition~\ref{prop:gate_averaging} identifies $\omega$ as the signed component preserved by example averaging. Thus output-preserving selection concentrates mean-aligned supervised sensitivity, not only large per-example second moments.

Activation energy is a strong baseline: its model-level median rank correlation with $\tau$ ranges from $0.85$ to $0.96$ (Table~\ref{tab:geometry-gradient-summary}). In the stricter all-decomposition same-$k$ control, downstream selections trail the activation-energy head by only $0.030$ of total $\tau$ mass and $0.026$ of total $\omega$ mass at the model-hierarchical median. They attain $0.909$ and $0.938$ median efficiency relative to the exact additive top-$k$ score oracles (Table~\ref{tab:same-budget-oracle}; Figure~\ref{fig:same_budget_mass}). This substantial, model-dependent oracle capture is notable because the masks optimize held-out output preservation, not either gate score; it connects behavioral sufficiency, activation energy, and supervised local sensitivity without claiming that downstream selection universally surpasses energy ranking.

The SVD endpoint analysis complements this set-level result: the normalized entropy-effective-number analogue of energy rank decreases from the earliest to final checkpoint in 36 of 42 model--layer pairs (Figure~\ref{fig:svd_tail}). Training therefore often concentrates activation energy while the behaviorally selected allocation continues to change, motivating the curvature-weighted view of Section~\ref{sec:theory}.

\subsection{Behavioral monitors agree broadly and add information}

The output metrics provide a final consistency check. Pooled $\Delta\mathrm{KL}$--$\Delta\mathrm{CE}$ associations are $0.96$--$0.97$ for autoregressive and $0.71$--$0.88$ for masked models, and CE--accuracy signs agree in $80.4\%$ of reconstruction and $89.6\%$ of downstream cells after model macro-averaging. Within checkpoint strata, KL--CE correlations range from $0.75$ to $0.95$ for autoregressive models and from $0.15$ to $0.55$ for masked models (Figure~\ref{fig:behavioral_agreement}); gold-label CE also improves in $8$--$16\%$ of masked-model interventions but only about $0$--$2\%$ of autoregressive interventions. The strong aggregate agreement supports the overall behavioral signal, while the stratified differences show the value of retaining both full-distribution and target-token monitors.

\section{Discussion}
\label{sec:discussion}

Three empirical findings stand out. First, direct downstream optimization recovers compact functional soft masks: a median $6.6\%$ soft occupancy preserves the original predictive distribution to relative KL $0.020$. Second, the temporal and depth profiles are structured and often agree across objectives within a model, while differing across models. Third, decomposition-specific transfer endpoints repeat across all seven models for SVD and ICA. Together, these results show that concept dynamics is more informative as a collection of target-aware profiles than as a single maturity score.

\paragraph{Interpretive intuition.}
The occupancy gap, substantial downstream NMSE, and small output shift in Table~\ref{tab:main-operating-points} are consistent, under the local approximation, with pronounced task-conditional anisotropy in the pullback predictive Fisher of the remaining network. Reconstruction weights residual energy uniformly, whereas Equation~\eqref{eq:reconstruction_anisotropy} weights it by this predictive curvature. Thus much activation variation may lie in weakly used directions, while a small retained soft mass carries a disproportionate share of supervised gate sensitivity. This is an intuition rather than a measured Fisher eigenspectrum: objective scaling, decomposition alignment, nonlocal interventions, and optimization may also contribute. It is nevertheless consistent with the tangent-aligned account of \citet{bernas2026revisitinganisotropy}, in which training preferentially reinforces directions carrying concentrated gradient energy and anisotropy.

The broad, behaviorally similar reconstruction and fixed-head operating points motivate a complementary sender--receiver hypothesis. Training may shape activation space so that useful distinctions are accessible to a shallow affine readout, while downstream computation acts as a direction-selective receiver that selects and amplifies a narrower set. A linearly accessible sender state coupled to such a receiver is one possible distributed form of a circuit. The present interventions do not identify that circuit or directly establish linear separability; for masked models, moreover, the affine translator still feeds a nonlinear MLM head. Testing this picture requires aligned concept identities and direct spectral measurements across checkpoints.

\FloatBarrier

\section{Conclusion}
\label{sec:conclusion}

We introduced checkpoint-wise interventions for reconstruction, fixed-head decodability, downstream behavior, and learned-alignment transfer. Across seven models, they identify compact output-preserving soft masks enriched for supervised sensitivity and structured temporal, depth, and decomposition profiles. The analytic relations explain why the monitors agree without being interchangeable. Concept dynamics thus turns static decompositions into functional longitudinal measurements with explicit preservation targets.

\section*{Limitations}

The study spans seven pretrained models with one run per model, and checkpoint, layer, method, and target cells are correlated repeated measurements rather than independent replications. Differences among model families jointly reflect pretraining data, tokenization, objective, architecture, and checkpoint schedule. Evaluation observes one target-token prediction per sequence; extending the conclusions to sequence generation, instruction following, in-context learning, multilingual downstream tasks, or substantially larger models requires dedicated evaluation.

Functional sufficiency is conditional on the chosen decomposition, preservation target, sparsity operating point, and intervention operator. It does not make a selected direction uniquely human-interpretable: redundant or distributed coordinates can preserve the same behavior. Masked reconstruction can also move an activation away from naturally occurring states, so an output change may combine removed information with sensitivity to intervention geometry. Soft occupancy should consequently be read as target-conditional retained capacity, not intrinsic dimensionality.

Checkpoint profiles summarize population-level effects on a scheduled grid rather than persistence of individual concept identities. Learned transfer alignment measures how well earlier codes can reconstruct a later activation through the fitted map; success can reflect alignment flexibility as well as continuity. Finally, the gate scores are local, supervised, first-order sensitivities: $\tau$ is a second moment and $\omega$ its mean-aligned component, not a realized training update or Fisher eigenvalue. These boundaries limit mechanistic interpretations while leaving the paired intervention comparisons intact.

\paragraph{Scope of the sufficiency claim.}
In this work, sufficiency is operational and target-conditional: a mask is evaluated by the preservation loss
obtained at the single regularization setting used in the experiments. We do not estimate a universal fidelity
threshold, a first checkpoint at which that threshold is crossed, or a persistent identity-matched component
trajectory. Checkpoint profiles should therefore be interpreted as changes in fitted preservation operating
points rather than as identity-tracked onset times.

\section*{Ethical Considerations}

This work analyzes public pretrained models and text corpora and does not involve human participants or deploy a new generative system. Its intended use is to improve understanding of learning and generalization through training-time monitoring. The source corpora and models may nevertheless inherit social biases, offensive material, or memorized personal information. Intervention fidelity is not a certificate of fairness, privacy, robustness, or safety and should not be used as one. The analysis reuses existing checkpoints but incurs additional activation, decomposition, and intervention compute.

\section*{Use of Large Language Models}

During preparation of this manuscript, OpenAI GPT-5.5 and GPT-5.6 were used as assistive tools for editing author-written prose, suggesting limited debugging or optimization changes, organizing large result tables, and suggesting missing citations. The model played no role in formulating the research hypotheses, designing the experiments, developing the original experimental codebase, or independently interpreting the results. All scientific decisions, methodologies, analyses, and conclusions were developed and verified by the authors.

\bibliography{custom}

@article{amari1998natural,
  title = {Natural Gradient Works Efficiently in Learning},
  author = {Amari, Shun-ichi},
  journal = {Neural Computation},
  volume = {10},
  number = {2},
  pages = {251--276},
  year = {1998},
  doi = {10.1162/089976698300017746},
  url = {https://doi.org/10.1162/089976698300017746}
}

@inproceedings{kingma2015adam,
  title = {{Adam}: A Method for Stochastic Optimization},
  author = {Kingma, Diederik P. and Ba, Jimmy},
  booktitle = {3rd International Conference on Learning Representations},
  year = {2015},
  url = {https://arxiv.org/abs/1412.6980}
}

@inproceedings{loshchilov2019decoupled,
  title = {Decoupled Weight Decay Regularization},
  author = {Loshchilov, Ilya and Hutter, Frank},
  booktitle = {7th International Conference on Learning Representations},
  publisher = {OpenReview.net},
  year = {2019},
  url = {https://openreview.net/forum?id=Bkg6RiCqY7}
}

@article{bricken2023monosemanticity,
  title = {Towards Monosemanticity: Decomposing Language Models With Dictionary Learning},
  author = {Bricken, Trenton and Templeton, Adly and Batson, Joshua and Chen, Brian and Jermyn, Adam and Conerly, Tom and Turner, Nick and Anil, Cem and Denison, Carson and Askell, Amanda and Lasenby, Robert and Wu, Yifan and Kravec, Shauna and Schiefer, Nicholas and Maxwell, Tim and Joseph, Nicholas and Hatfield-Dodds, Zac and Tamkin, Alex and Nguyen, Karina and McLean, Brayden and Burke, Josiah E. and Hume, Tristan and Carter, Shan and Henighan, Tom and Olah, Christopher},
  year = {2023},
  journal = {Transformer Circuits Thread},
  url = {https://transformer-circuits.pub/2023/monosemantic-features/}
}

@inproceedings{collins2018deep,
  title = {Deep Feature Factorization for Concept Discovery},
  author = {Collins, Edo and Achanta, Radhakrishna and S{\"u}sstrunk, Sabine},
  booktitle = {Computer Vision -- ECCV 2018},
  series = {Lecture Notes in Computer Science},
  volume = {11218},
  pages = {352--368},
  year = {2018},
  publisher = {Springer},
  doi = {10.1007/978-3-030-01264-9_21},
  url = {https://doi.org/10.1007/978-3-030-01264-9_21}
}

@misc{bernas2026revisitinganisotropy,
  title = {Revisiting Anisotropy in Language Transformers: The Geometry of Learning Dynamics},
  author = {Bernas, Raphael and Jourdan, Fanny and Poch{\'e}, Antonin and Hudelot, C{\'e}line},
  year = {2026},
  eprint = {2604.08764},
  archivePrefix = {arXiv},
  primaryClass = {cs.CL},
  note = {Accepted at the 43rd International Conference on Machine Learning (ICML 2026)},
  doi = {10.48550/arXiv.2604.08764},
  url = {https://arxiv.org/abs/2604.08764}
}

@inproceedings{braun2024functionally,
  title = {Identifying Functionally Important Features with End-to-End Sparse Dictionary Learning},
  author = {Braun, Dan and Taylor, Jordan and Goldowsky-Dill, Nicholas and Sharkey, Lee},
  booktitle = {Advances in Neural Information Processing Systems},
  volume = {37},
  pages = {107286--107325},
  year = {2024},
  doi = {10.52202/079017-3408},
  url = {https://proceedings.neurips.cc/paper_files/paper/2024/hash/c212c1b88395ab68d5e1671c17883ec6-Abstract-Conference.html}
}

@article{bair2006prediction,
  title = {Prediction by Supervised Principal Components},
  author = {Bair, Eric and Hastie, Trevor and Paul, Debashis and Tibshirani, Robert},
  journal = {Journal of the American Statistical Association},
  volume = {101},
  number = {473},
  pages = {119--137},
  year = {2006},
  doi = {10.1198/016214505000000628},
  url = {https://doi.org/10.1198/016214505000000628}
}

@article{eckart1936approximation,
  title = {The Approximation of One Matrix by Another of Lower Rank},
  author = {Eckart, Carl and Young, Gale},
  journal = {Psychometrika},
  volume = {1},
  number = {3},
  pages = {211--218},
  year = {1936},
  doi = {10.1007/BF02288367},
  url = {https://doi.org/10.1007/BF02288367}
}

@inproceedings{liu2021group,
  title = {Group {Fisher} Pruning for Practical Network Compression},
  author = {Liu, Liyang and Zhang, Shilong and Kuang, Zhanghui and Zhou, Aojun and Xue, Jing-Hao and Wang, Xinjiang and Chen, Yimin and Yang, Wenming and Liao, Qingmin and Zhang, Wayne},
  booktitle = {Proceedings of the 38th International Conference on Machine Learning},
  pages = {7021--7032},
  year = {2021},
  volume = {139},
  series = {Proceedings of Machine Learning Research},
  publisher = {PMLR},
  url = {https://proceedings.mlr.press/v139/liu21ab.html}
}

@inproceedings{roy2007effective,
  title = {The Effective Rank: A Measure of Effective Dimensionality},
  author = {Roy, Olivier and Vetterli, Martin},
  booktitle = {15th European Signal Processing Conference},
  pages = {606--610},
  year = {2007},
  publisher = {EURASIP},
  url = {https://eurasip.org/Proceedings/Eusipco/Eusipco2007/Papers/a5p-h05.pdf}
}

@misc{chatterjee2020coherence,
  title = {Making Coherence Out of Nothing At All: Measuring the Evolution of Gradient Alignment},
  author = {Chatterjee, Satrajit and Zielinski, Piotr},
  year = {2020},
  eprint = {2008.01217},
  archivePrefix = {arXiv},
  primaryClass = {cs.LG},
  doi = {10.48550/arXiv.2008.01217},
  url = {https://arxiv.org/abs/2008.01217}
}

@inproceedings{cunningham2023sparse,
  title = {Sparse Autoencoders Find Highly Interpretable Features in Language Models},
  author = {Huben, Robert and Cunningham, Hoagy and Smith, Logan Riggs and Ewart, Aidan and Sharkey, Lee},
  booktitle = {The Twelfth International Conference on Learning Representations},
  year = {2024},
  url = {https://openreview.net/forum?id=F76bwRSLeK}
}

@inproceedings{biderman2023pythia,
  title = {Pythia: A Suite for Analyzing Large Language Models Across Training and Scaling},
  author = {Biderman, Stella and Schoelkopf, Hailey and Anthony, Quentin Gregory and Bradley, Herbie and O'Brien, Kyle and Hallahan, Eric and Khan, Mohammad Aflah and Purohit, Shivanshu and Prashanth, Usvsn Sai and Raff, Edward and Skowron, Aviya and Sutawika, Lintang and Van Der Wal, Oskar},
  booktitle = {Proceedings of the 40th International Conference on Machine Learning},
  pages = {2397--2430},
  year = {2023},
  volume = {202},
  series = {Proceedings of Machine Learning Research},
  publisher = {PMLR},
  url = {https://proceedings.mlr.press/v202/biderman23a.html}
}

@article{ding2010convex,
  title = {Convex and Semi-Nonnegative Matrix Factorizations},
  author = {Ding, Chris H. Q. and Li, Tao and Jordan, Michael I.},
  journal = {IEEE Transactions on Pattern Analysis and Machine Intelligence},
  volume = {32},
  number = {1},
  pages = {45--55},
  year = {2010},
  doi = {10.1109/TPAMI.2008.277},
  url = {https://doi.org/10.1109/TPAMI.2008.277}
}

@inproceedings{diehl-martinez-etal-2024-tending,
  title = {Tending Towards Stability: Convergence Challenges in Small Language Models},
  author = {Diehl Martinez, Richard and Lesci, Pietro and Buttery, Paula},
  booktitle = {Findings of the Association for Computational Linguistics: EMNLP 2024},
  pages = {3275--3286},
  year = {2024},
  address = {Miami, Florida, USA},
  publisher = {Association for Computational Linguistics},
  doi = {10.18653/v1/2024.findings-emnlp.187},
  url = {https://aclanthology.org/2024.findings-emnlp.187/}
}

@article{elazar-etal-2021-amnesic,
  title = {Amnesic Probing: Behavioral Explanation with Amnesic Counterfactuals},
  author = {Elazar, Yanai and Ravfogel, Shauli and Jacovi, Alon and Goldberg, Yoav},
  journal = {Transactions of the Association for Computational Linguistics},
  volume = {9},
  pages = {160--175},
  year = {2021},
  publisher = {MIT Press},
  doi = {10.1162/tacl_a_00359},
  url = {https://aclanthology.org/2021.tacl-1.10/}
}

@article{hyvarinen2000independent,
  title = {Independent Component Analysis: Algorithms and Applications},
  author = {Hyv{\"a}rinen, Aapo and Oja, Erkki},
  journal = {Neural Networks},
  volume = {13},
  number = {4--5},
  pages = {411--430},
  year = {2000},
  doi = {10.1016/S0893-6080(00)00026-5},
  url = {https://doi.org/10.1016/S0893-6080(00)00026-5}
}

@article{facco2017estimating,
  title = {Estimating the Intrinsic Dimension of Datasets by a Minimal Neighborhood Information},
  author = {Facco, Elena and d'Errico, Maria and Rodriguez, Alex and Laio, Alessandro},
  journal = {Scientific Reports},
  volume = {7},
  number = {1},
  pages = {12140},
  year = {2017},
  doi = {10.1038/s41598-017-11873-y},
  url = {https://www.nature.com/articles/s41598-017-11873-y}
}

@inproceedings{fel2023craft,
  title = {{CRAFT}: Concept Recursive Activation FacTorization for Explainability},
  author = {Fel, Thomas and Picard, Agustin and B{\'e}thune, Louis and Boissin, Thibaut and Vigouroux, David and Colin, Julien and Cad{\`e}ne, R{\'e}mi and Serre, Thomas},
  booktitle = {Proceedings of the IEEE/CVF Conference on Computer Vision and Pattern Recognition},
  pages = {2711--2721},
  year = {2023},
  publisher = {IEEE},
  doi = {10.1109/CVPR52729.2023.00266},
  url = {https://openaccess.thecvf.com/content/CVPR2023/html/Fel_CRAFT_Concept_Recursive_Activation_FacTorization_for_Explainability_CVPR_2023_paper.html}
}

@inproceedings{fel2023holistic,
  title = {A Holistic Approach to Unifying Automatic Concept Extraction and Concept Importance Estimation},
  author = {Fel, Thomas and Boutin, Victor and B{\'e}thune, Louis and Cad{\`e}ne, R{\'e}mi and Moayeri, Mazda and And{\'e}ol, L{\'e}o and Chalvidal, Mathieu and Serre, Thomas},
  booktitle = {Advances in Neural Information Processing Systems},
  volume = {36},
  pages = {54805--54818},
  year = {2023},
  publisher = {Curran Associates, Inc.},
  doi = {10.52202/075280-2391},
  url = {https://proceedings.neurips.cc/paper_files/paper/2023/hash/abf3682c9cf9245a0294a4bebe4544ff-Abstract-Conference.html}
}

@book{jolliffe2002principal,
  title = {Principal Component Analysis},
  author = {Jolliffe, Ian T.},
  edition = {2nd},
  year = {2002},
  publisher = {Springer},
  address = {New York},
  doi = {10.1007/b98835},
  url = {https://link.springer.com/book/10.1007/b98835}
}

@inproceedings{jourdan-etal-2023-cockatiel,
  title = {{COCKATIEL}: {CO}ntinuous Concept ran{K}ed {AT}tribution with Interpretable {EL}ements for explaining neural net classifiers on {NLP}},
  author = {Jourdan, Fanny and Picard, Agustin and Fel, Thomas and Risser, Laurent and Loubes, Jean-Michel and Asher, Nicholas},
  booktitle = {Findings of the Association for Computational Linguistics: ACL 2023},
  pages = {5120--5136},
  year = {2023},
  address = {Toronto, Canada},
  publisher = {Association for Computational Linguistics},
  doi = {10.18653/v1/2023.findings-acl.317},
  url = {https://aclanthology.org/2023.findings-acl.317/}
}

@article{geiger2025causal,
  title = {Causal Abstraction: A Theoretical Foundation for Mechanistic Interpretability},
  author = {Geiger, Atticus and Ibeling, Duligur and Zur, Amir and Chaudhary, Maheep and Chauhan, Sonakshi and Huang, Jing and Arora, Aryaman and Wu, Zhengxuan and Goodman, Noah and Potts, Christopher and Icard, Thomas},
  journal = {Journal of Machine Learning Research},
  volume = {26},
  number = {83},
  pages = {1--64},
  year = {2025},
  url = {https://www.jmlr.org/papers/v26/23-0058.html}
}

@article{graziani2023uncovering,
  title = {Uncovering Unique Concept Vectors through Latent Space Decomposition},
  author = {Graziani, Mara and O'Mahony, Laura and Nguyen, An-Phi and M{\"u}ller, Henning and Andrearczyk, Vincent},
  journal = {Transactions on Machine Learning Research},
  year = {2023},
  issn = {2835-8856},
  url = {https://openreview.net/forum?id=LT4DXqUJTD}
}

@inproceedings{nanda2023progress,
  title = {Progress Measures for Grokking via Mechanistic Interpretability},
  author = {Nanda, Neel and Chan, Lawrence and Lieberum, Tom and Smith, Jess and Steinhardt, Jacob},
  booktitle = {International Conference on Learning Representations},
  year = {2023},
  url = {https://openreview.net/forum?id=9XFSbDPmdW}
}

@article{martens2020new,
  title = {New Insights and Perspectives on the Natural Gradient Method},
  author = {Martens, James},
  journal = {Journal of Machine Learning Research},
  volume = {21},
  number = {146},
  pages = {1--76},
  year = {2020},
  url = {https://jmlr.csail.mit.edu/papers/v21/17-678.html}
}

@inproceedings{olmo2025features,
  title = {Features that Make a Difference: Leveraging Gradients for Improved Dictionary Learning},
  author = {Olmo, Jeffrey and Wilson, Jared and Forsey, Max and Hepner, Bryce and Howe, Thomas Vincent and Wingate, David},
  booktitle = {Findings of the Association for Computational Linguistics: NAACL 2025},
  pages = {7624--7634},
  year = {2025},
  address = {Albuquerque, New Mexico},
  publisher = {Association for Computational Linguistics},
  doi = {10.18653/v1/2025.findings-naacl.423},
  url = {https://aclanthology.org/2025.findings-naacl.423/}
}

@inproceedings{hewitt-liang-2019-designing,
  title = {Designing and Interpreting Probes with Control Tasks},
  author = {Hewitt, John and Liang, Percy},
  booktitle = {Proceedings of the 2019 Conference on Empirical Methods in Natural Language Processing and the 9th International Joint Conference on Natural Language Processing},
  pages = {2733--2743},
  year = {2019},
  address = {Hong Kong, China},
  publisher = {Association for Computational Linguistics},
  doi = {10.18653/v1/D19-1275},
  url = {https://aclanthology.org/D19-1275/}
}

@inproceedings{kim2018interpretability,
  title = {Interpretability Beyond Feature Attribution: Quantitative Testing with Concept Activation Vectors ({TCAV})},
  author = {Kim, Been and Wattenberg, Martin and Gilmer, Justin and Cai, Carrie and Wexler, James and Vi{\'e}gas, Fernanda and Sayres, Rory},
  booktitle = {Proceedings of the 35th International Conference on Machine Learning},
  pages = {2668--2677},
  year = {2018},
  volume = {80},
  series = {Proceedings of Machine Learning Research},
  publisher = {PMLR},
  url = {https://proceedings.mlr.press/v80/kim18d.html}
}

@inproceedings{kornblith2019similarity,
  title = {Similarity of Neural Network Representations Revisited},
  author = {Kornblith, Simon and Norouzi, Mohammad and Lee, Honglak and Hinton, Geoffrey},
  booktitle = {Proceedings of the 36th International Conference on Machine Learning},
  pages = {3519--3529},
  year = {2019},
  volume = {97},
  series = {Proceedings of Machine Learning Research},
  publisher = {PMLR},
  url = {https://proceedings.mlr.press/v97/kornblith19a.html}
}

@inproceedings{kunstner2019limitations,
  title = {Limitations of the Empirical {Fisher} Approximation for Natural Gradient Descent},
  author = {Kunstner, Frederik and Balles, Lukas and Hennig, Philipp},
  booktitle = {Advances in Neural Information Processing Systems},
  volume = {32},
  year = {2019},
  url = {https://proceedings.neurips.cc/paper/2019/hash/46a558d97954d0692411c861cf78ef79-Abstract.html}
}

@article{lee1999learning,
  title = {Learning the Parts of Objects by Non-Negative Matrix Factorization},
  author = {Lee, Daniel D. and Seung, H. Sebastian},
  journal = {Nature},
  volume = {401},
  pages = {788--791},
  year = {1999},
  doi = {10.1038/44565},
  url = {https://www.nature.com/articles/44565}
}

@inproceedings{poche-etal-2025-consim,
  title = {{ConSim}: Measuring Concept-Based Explanations' Effectiveness with Automated Simulatability},
  author = {Poch{\'e}, Antonin and Jacovi, Alon and Picard, Agustin Martin and Boutin, Victor and Jourdan, Fanny},
  booktitle = {Proceedings of the 63rd Annual Meeting of the Association for Computational Linguistics (Volume 1: Long Papers)},
  pages = {5594--5615},
  year = {2025},
  address = {Vienna, Austria},
  publisher = {Association for Computational Linguistics},
  doi = {10.18653/v1/2025.acl-long.279},
  url = {https://aclanthology.org/2025.acl-long.279/}
}

@inproceedings{raghu2017svcca,
  title = {{SVCCA}: Singular Vector Canonical Correlation Analysis for Deep Learning Dynamics and Interpretability},
  author = {Raghu, Maithra and Gilmer, Justin and Yosinski, Jason and Sohl-Dickstein, Jascha},
  booktitle = {Advances in Neural Information Processing Systems},
  volume = {30},
  pages = {6076--6085},
  year = {2017},
  url = {https://dl.acm.org/doi/10.5555/3295222.3295356}
}

@inproceedings{razzhigaev2024shapelearninganisotropyintrinsic,
  title = {The Shape of Learning: Anisotropy and Intrinsic Dimensions in Transformer-Based Models},
  author = {Razzhigaev, Anton and Mikhalchuk, Matvey and Goncharova, Elizaveta and Oseledets, Ivan and Dimitrov, Denis and Kuznetsov, Andrey},
  booktitle = {Findings of the Association for Computational Linguistics: EACL 2024},
  pages = {868--874},
  year = {2024},
  address = {St. Julian's, Malta},
  publisher = {Association for Computational Linguistics},
  doi = {10.18653/v1/2024.findings-eacl.58},
  url = {https://aclanthology.org/2024.findings-eacl.58/}
}

@inproceedings{syed2024attribution,
  title = {Attribution Patching Outperforms Automated Circuit Discovery},
  author = {Syed, Aaquib and Rager, Can and Conmy, Arthur},
  booktitle = {Proceedings of the 7th BlackboxNLP Workshop: Analyzing and Interpreting Neural Networks for NLP},
  pages = {407--416},
  year = {2024},
  address = {Miami, Florida, US},
  publisher = {Association for Computational Linguistics},
  doi = {10.18653/v1/2024.blackboxnlp-1.25},
  url = {https://aclanthology.org/2024.blackboxnlp-1.25/}
}

@inproceedings{inaba-etal-2025-bilingual,
  title = {How a Bilingual {LM} Becomes Bilingual: Tracing Internal Representations with Sparse Autoencoders},
  author = {Inaba, Tatsuro and Kamoda, Go and Inui, Kentaro and Isonuma, Masaru and Miyao, Yusuke and Oseki, Yohei and Takagi, Yu and Heinzerling, Benjamin},
  booktitle = {Findings of the Association for Computational Linguistics: EMNLP 2025},
  pages = {13458--13470},
  year = {2025},
  address = {Suzhou, China},
  publisher = {Association for Computational Linguistics},
  doi = {10.18653/v1/2025.findings-emnlp.725},
  url = {https://aclanthology.org/2025.findings-emnlp.725/}
}

@misc{wang2026fishback,
  title = {{FishBack}: Pullback {Fisher} Geometry for Optimal Activation Steering in Transformers},
  author = {Wang, Sihan and Zhao, Jiayi},
  year = {2026},
  eprint = {2605.17231},
  archivePrefix = {arXiv},
  primaryClass = {cs.LG},
  doi = {10.48550/arXiv.2605.17231},
  url = {https://arxiv.org/abs/2605.17231}
}

@misc{xu2024trackingdynamics,
  title = {Tracking the Feature Dynamics in {LLM} Training: A Mechanistic Study},
  author = {Xu, Yang and Wang, Yi and Huang, Hengguan and Wang, Hao},
  year = {2024},
  eprint = {2412.17626},
  archivePrefix = {arXiv},
  primaryClass = {cs.LG},
  doi = {10.48550/arXiv.2412.17626},
  url = {https://arxiv.org/abs/2412.17626}
}

@article{zhang2021invertible,
  title = {Invertible Concept-based Explanations for {CNN} Models with Non-negative Concept Activation Vectors},
  author = {Zhang, Ruihan and Madumal, Prashan and Miller, Tim and Ehinger, Krista A. and Rubinstein, Benjamin I. P.},
  journal = {Proceedings of the AAAI Conference on Artificial Intelligence},
  volume = {35},
  number = {13},
  pages = {11682--11690},
  year = {2021},
  doi = {10.1609/aaai.v35i13.17389},
  url = {https://ojs.aaai.org/index.php/AAAI/article/view/17389}
}

@inproceedings{groeneveld-etal-2024-olmo,
  title = {{OLM}o: Accelerating the Science of Language Models},
  author = {Groeneveld, Dirk and Beltagy, Iz and Walsh, Evan and Bhagia, Akshita and Kinney, Rodney and Tafjord, Oyvind and Jha, Ananya and Ivison, Hamish and Magnusson, Ian and Wang, Yizhong and Arora, Shane and Atkinson, David and Authur, Russell and Chandu, Khyathi and Cohan, Arman and Dumas, Jennifer and Elazar, Yanai and Gu, Yuling and Hessel, Jack and Khot, Tushar and Merrill, William and Morrison, Jacob and Muennighoff, Niklas and Naik, Aakanksha and Nam, Crystal and Peters, Matthew and Pyatkin, Valentina and Ravichander, Abhilasha and Schwenk, Dustin and Shah, Saurabh and Smith, William and Strubell, Emma and Subramani, Nishant and Wortsman, Mitchell and Dasigi, Pradeep and Lambert, Nathan and Richardson, Kyle and Zettlemoyer, Luke and Dodge, Jesse and Lo, Kyle and Soldaini, Luca and Smith, Noah and Hajishirzi, Hannaneh},
  booktitle = {Proceedings of the 62nd Annual Meeting of the Association for Computational Linguistics (Volume 1: Long Papers)},
  pages = {15789--15809},
  year = {2024},
  address = {Bangkok, Thailand},
  publisher = {Association for Computational Linguistics},
  doi = {10.18653/v1/2024.acl-long.841},
  url = {https://aclanthology.org/2024.acl-long.841/}
}

@misc{allenai2024olmo1b0724,
  author = {{Allen Institute for AI}},
  title = {Model Card for {OLMo} {1B} {July} 2024},
  year = {2024},
  howpublished = {Hugging Face model repository},
  url = {https://huggingface.co/allenai/OLMo-1B-0724-hf},
  note = {Repository identifier: allenai/OLMo-1B-0724-hf; last accessed 1 August 2026}
}

@misc{eleutherai2023pythia70m,
  author = {{EleutherAI}},
  title = {Model Card for {Pythia-70M}},
  year = {2023},
  howpublished = {Hugging Face model repository},
  url = {https://huggingface.co/EleutherAI/pythia-70m},
  note = {Repository identifier: EleutherAI/pythia-70m; last accessed 1 August 2026}
}

@misc{eleutherai2023pythia160m,
  author = {{EleutherAI}},
  title = {Model Card for {Pythia-160M}},
  year = {2023},
  howpublished = {Hugging Face model repository},
  url = {https://huggingface.co/EleutherAI/pythia-160m},
  note = {Repository identifier: EleutherAI/pythia-160m; last accessed 1 August 2026}
}

@misc{eleutherai2023pythia410m,
  author = {{EleutherAI}},
  title = {Model Card for {Pythia-410M}},
  year = {2023},
  howpublished = {Hugging Face model repository},
  url = {https://huggingface.co/EleutherAI/pythia-410m},
  note = {Repository identifier: EleutherAI/pythia-410m; last accessed 1 August 2026}
}

@inproceedings{boizard2025eurobert,
  title = {{EuroBERT}: Scaling Multilingual Encoders for {European} Languages},
  author = {Boizard, Nicolas and Gisserot-Boukhlef, Hippolyte and Alves, Duarte M. and Martins, Andr{\'e} F. T. and Hammal, Ayoub and Corro, Caio and Hudelot, C{\'e}line and Malherbe, Emmanuel and Malaboeuf, Etienne and Jourdan, Fanny and Hautreux, Gabriel and Alves, Jo{\~a}o and El-Haddad, Kevin and Faysse, Manuel and Peyrard, Maxime and Guerreiro, Nuno M. and Fernandes, Patrick and Rei, Ricardo and Colombo, Pierre},
  booktitle = {Second Conference on Language Modeling},
  pages = {1--28},
  year = {2025},
  address = {Montreal, Canada},
  url = {https://openreview.net/forum?id=jdOC24msVq}
}

@inproceedings{antoun-etal-2025-modernbert,
  title = {{M}odern{BERT} or {D}e{BERT}a{V}3? Examining Architecture and Data Influence on Transformer Encoder Models Performance},
  author = {Antoun, Wissam and Sagot, Beno{\^i}t and Seddah, Djam{\'e}},
  booktitle = {Proceedings of the 14th International Joint Conference on Natural Language Processing and the 4th Conference of the Asia-Pacific Chapter of the Association for Computational Linguistics},
  pages = {3061--3074},
  year = {2025},
  address = {Mumbai, India},
  publisher = {The Asian Federation of Natural Language Processing and The Association for Computational Linguistics},
  doi = {10.18653/v1/2025.ijcnlp-long.164},
  url = {https://aclanthology.org/2025.ijcnlp-long.164/}
}

@misc{eurobert2025modelcard210m,
  author = {{EuroBERT Team}},
  title = {Model Card for {EuroBERT-210m}},
  year = {2025},
  howpublished = {Hugging Face model repository},
  url = {https://huggingface.co/EuroBERT/EuroBERT-210m},
  note = {Repository identifier: EuroBERT/EuroBERT-210m; last accessed 1 August 2026}
}

@misc{eurobert2025modelcard610m,
  author = {{EuroBERT Team}},
  title = {Model Card for {EuroBERT-610m}},
  year = {2025},
  howpublished = {Hugging Face model repository},
  url = {https://huggingface.co/EuroBERT/EuroBERT-610m},
  note = {Repository identifier: EuroBERT/EuroBERT-610m; last accessed 1 August 2026}
}

@misc{almanach2025moderncamembertckpts,
  author = {{ALMAnaCH, Inria}},
  title = {{ModernCamemBERT} Checkpoint Repository},
  year = {2025},
  howpublished = {Hugging Face model repository},
  url = {https://huggingface.co/almanach/moderncamembert-base-ckpts},
  note = {Repository identifier: almanach/moderncamembert-base-ckpts; the repository provides no model card; last accessed 1 August 2026}
}

@misc{almanach2025moderncamembert,
  author = {{ALMAnaCH, Inria}},
  title = {Model Card for {ModernCamemBERT-base}},
  year = {2025},
  howpublished = {Hugging Face model repository},
  url = {https://huggingface.co/almanach/moderncamembert-base},
  note = {Repository identifier: almanach/moderncamembert-base; last accessed 1 August 2026}
}

@article{raffel2020t5,
  author = {Raffel, Colin and Shazeer, Noam and Roberts, Adam and Lee, Katherine and Narang, Sharan and Matena, Michael and Zhou, Yanqi and Li, Wei and Liu, Peter J.},
  title = {Exploring the Limits of Transfer Learning with a Unified Text-to-Text Transformer},
  journal = {Journal of Machine Learning Research},
  year = {2020},
  volume = {21},
  number = {140},
  pages = {1--67},
  url = {https://jmlr.org/papers/v21/20-074.html}
}

@inproceedings{merity2017pointer,
  title = {Pointer Sentinel Mixture Models},
  author = {Merity, Stephen and Xiong, Caiming and Bradbury, James and Socher, Richard},
  booktitle = {International Conference on Learning Representations},
  year = {2017},
  url = {https://openreview.net/forum?id=Byj72udxe}
}

@misc{allenai2026c4card,
  author = {{Allen Institute for AI}},
  title = {Dataset Card for {C4}},
  year = {2026},
  howpublished = {Hugging Face dataset repository},
  url = {https://huggingface.co/datasets/allenai/c4},
  note = {Repository identifier: allenai/c4; last accessed 1 August 2026}
}

@misc{salesforce2026wikitextcard,
  author = {{Salesforce AI Research}},
  title = {Dataset Card for {WikiText}},
  year = {2026},
  howpublished = {Hugging Face dataset repository},
  url = {https://huggingface.co/datasets/Salesforce/wikitext},
  note = {Repository identifier: Salesforce/wikitext; last accessed 1 August 2026}
}

@inproceedings{penedo2024fineweb,
  author = {Penedo, Guilherme and Kydl{\'i}\v{c}ek, Hynek and Ben Allal, Loubna and Lozhkov, Anton and Mitchell, Margaret and Raffel, Colin and von Werra, Leandro and Wolf, Thomas},
  title = {The {FineWeb} Datasets: Decanting the Web for the Finest Text Data at Scale},
  booktitle = {Advances in Neural Information Processing Systems},
  volume = {37},
  pages = {30811--30849},
  year = {2024},
  publisher = {Curran Associates, Inc.},
  doi = {10.52202/079017-0970},
  url = {https://proceedings.neurips.cc/paper_files/paper/2024/hash/370df50ccfdf8bde18f8f9c2d9151bda-Abstract-Datasets_and_Benchmarks_Track.html}
}

@misc{lozhkov2024fineweb-edu,
  author = {Lozhkov, Anton and Ben Allal, Loubna and von Werra, Leandro and Wolf, Thomas},
  title = {{FineWeb-Edu}: The Finest Collection of Educational Content},
  year = {2024},
  publisher = {Hugging Face},
  doi = {10.57967/hf/2497},
  url = {https://huggingface.co/datasets/HuggingFaceFW/fineweb-edu},
  note = {Dataset card, repository identifier: HuggingFaceFW/fineweb-edu; last accessed 1 August 2026}
}

@misc{wikimedia2023wikipedia,
  author = {{Wikimedia Foundation}},
  title = {Dataset Card for {Wikipedia}},
  year = {2023},
  howpublished = {Hugging Face dataset repository},
  url = {https://huggingface.co/datasets/wikimedia/wikipedia},
  note = {Repository identifier: wikimedia/wikipedia; configurations 20231101.en and 20231101.fr, derived from \url{https://dumps.wikimedia.org/}; last accessed 1 August 2026}
}

@misc{mccandlish2018empirical,
  title = {An Empirical Model of Large-Batch Training},
  author = {McCandlish, Sam and Kaplan, Jared and Amodei, Dario and {OpenAI Dota Team}},
  year = {2018},
  eprint = {1812.06162},
  archivePrefix = {arXiv},
  primaryClass = {cs.LG},
  doi = {10.48550/arXiv.1812.06162},
  url = {https://arxiv.org/abs/1812.06162}
}

\appendix

\setcounter{figure}{0}
\renewcommand{\thefigure}{A\arabic{figure}}
\setcounter{table}{0}
\renewcommand{\thetable}{A\arabic{table}}

\section{Experimental Conditions}
\label{app:experimental_conditions}

\paragraph{Models, selection, and checkpoint grid.}
The exact artifacts are \texttt{EleutherAI/pythia-\{70m,160m,410m\}} \citep{biderman2023pythia,eleutherai2023pythia70m,eleutherai2023pythia160m,eleutherai2023pythia410m}, \texttt{allenai/OLMo-1B-0724-hf} \citep{groeneveld-etal-2024-olmo,allenai2024olmo1b0724}, \texttt{EuroBERT/EuroBERT-\{210m,610m\}} \citep{boizard2025eurobert,eurobert2025modelcard210m,eurobert2025modelcard610m}, and \texttt{almanach/moderncamembert-base-ckpts} \citep{antoun-etal-2025-modernbert,almanach2025moderncamembertckpts}. The Pythia sizes, EuroBERT sizes, and OLMo were selected before intervention effects were inspected; ModernCamemBERT was the first prespecified technical contingency. Both the primary and contingency runs completed, and all seven are reported. No model was retained or discarded because of its measured effect.

\begin{table*}[t]
\centering
\small
\setlength{\tabcolsep}{3pt}
\begin{tabular}{@{}llp{6.25cm}p{3.65cm}r@{}}
\toprule
Model & LM type & Checkpoints (thousands of updates) & MLP block indices & $K_{\max}$ \\
\midrule
Pythia-70M  & causal & 1, 10, 20, 29, 39, 48, 58, 67, 77, 86, 96, 105, 115, 124, 134, 143 & 0, 1, 2, 3, 4, 5 & 257 \\
Pythia-160M & causal & 1, 20, 39, 58, 86, 105, 124, 143 & 0, 2, 5, 6, 9, 11 & 385 \\
Pythia-410M & causal & 10, 26, 46, 62, 86, 102, 122, 138 & 0, 3, 10, 13, 20, 23 & 513 \\
OLMo-1B & causal & 1, 195, 388, 582, 873, 1067, 1260, 1454 & 0, 2, 6, 9, 13, 15 & 1025 \\
EuroBERT-210M & masked & 10, 70, 140, 200, 290, 350, 420, 480 & 0, 2, 5, 6, 9, 11 & 385 \\
EuroBERT-610M & masked & 10, 70, 140, 200, 290, 350, 420, 480 & 0, 4, 11, 14, 21, 25 & 577 \\
ModernCamem\-BERT & masked & 5, 15, 25, 60, 75, 115, 165, 211.961 & 0, 3, 9, 12, 18, 21 & 385 \\
\bottomrule
\end{tabular}
\caption{Executed model, checkpoint, layer, and candidate-budget grid. Checkpoints denote optimizer updates, not a common number of tokens or common training fraction.}
\label{tab:experimental_grid}
\end{table*}

Six MLP outputs, including both depth endpoints, are sampled per model. Pythia-70M was run at 16 checkpoints; balanced analyses use its endpoint-preserving subset $\{1,20,39,58,86,105,124,143\}$ so that every model contributes eight checkpoint ranks. The full bundle contains 6,144 intervention cells and 726,528 concept-score rows; the common grid contains 5,376 cells and 677,184 concept-score rows. Artifacts are identified by repository and checkpoint names rather than immutable repository commits. For ModernCamemBERT, the tokenizer from the \texttt{ep0-ba10000} artifact is held fixed across checkpoints and null local RoPE-base fields are restored to the training value 10,000; model weights are unchanged.

\paragraph{Data and evaluated rows.}
We stream the first requested nonempty rows from five sources: 4,000 from C4 English \citep{raffel2020t5,allenai2026c4card}, 3,000 from WikiText-103 \citep{merity2017pointer,salesforce2026wikitextcard}, 4,000 from FineWeb-Edu \citep{penedo2024fineweb,lozhkov2024fineweb-edu}, and 2,000 each from the 2023-11-01 English and French Wikipedia configurations \citep{wikimedia2023wikipedia}. Each source prefix is shuffled and partitioned separately in the same 8:4:3 ratio, using base seed 13 with deterministic source-specific offsets; the partitions are then pooled into disjoint train/calibration/test splits of 8,000/4,000/3,000 sequences. This prefix-based design is fully specified by the configurations, requested counts, seeds, and offsets; separate remote-snapshot and sampled-row manifests were not retained.

Inputs are truncated to 256 tokens, and one activation row is retained for each eligible sequence. For causal LMs, this is the final valid next-token prediction position. For masked LMs, a non-special position is chosen deterministically from the sample identifier and seed and replaced by the mask token. No prompts, demonstrations, or in-context examples are used; the number of in-context examples is zero.

\paragraph{Concept extraction and candidate budget.}
Dictionaries are fitted independently for every model, checkpoint, layer, and method using the available eligible training activation rows. We use uncentered truncated SVD; FastICA with unit-variance whitening and at most 1,000 iterations; NMF with a train-fitted nonnegative shift, NNDSVDa initialization, coordinate descent, and at most 1,000 iterations; and SemiNMF with nonnegative codes, at most 200 alternating updates, and tolerance $10^{-5}$. Randomized procedures use seed 13, and dictionary atoms are unit-normalized with codes rescaled accordingly. TwoNN uses at most 5,000 training rows per checkpoint/layer. Every stored estimate reached the $d_\ell/4$ lower clip in Equation~\eqref{eq:kmax}; the reported $K_{\max}$ values are consequently width-determined.

\paragraph{Mask fitting.}
No hyperparameter search, early stopping, or best-checkpoint selection was performed. All values below were fixed before the reported comparisons; optimizing one mask in each calibration cell is part of the estimator, not hyperparameter selection. Soft gates start from $\alpha_i=0$ and use Adam with learning rate $0.05$, a geometric temperature schedule from $1.0$ to $0.1$, and the unnormalized penalty $10^{-3}\sum_i z_i$. Reconstruction, linear, and transfer masks use 200 calibration updates. For transfer, the distribution objective uses $\mu=1.0$ on the code-distribution KL term; checkpoint alignment uses AdamW with learning rate $10^{-3}$, $\rho=10^{-4}$ decoupled weight decay, and 200 epochs. The memory-safe downstream optimizer uses 100 deterministic cyclic mini-batch updates of 512 calibration rows, with smaller model micro-batches. Held-out interventions use the soft gates; $K_{\mathrm{hard}}$ uses $z_i>0.5$. The auxiliary linear map is trained for 50 epochs with AdamW at learning rate $10^{-3}$ and decoupled weight decay $10^{-4}$.

\paragraph{Scores, evaluation, and numerical conditions.}
Concept scores use the eligible training rows and the gold-label activation gradient; the score manifests contain 7,995 rows per causal model and 8,000 per masked model. For component $i$, its stored activation energy on the $n$ scored rows is exactly
\[
e_i=\frac{1}{n}\lVert U_{:,i}\rVert_2^2\lVert v_i\rVert_2^2
=\frac{1}{n}\lVert U_{:,i}v_i^\top\rVert_F^2.
\]
Although dictionary atoms are numerically unit-normalized and the codes rescaled accordingly, the implementation retains the factor $\lVert v_i\rVert_2^2$; we therefore use the full expression above rather than replacing it by $n^{-1}\lVert U_{:,i}\rVert_2^2$. The other stored statistics are the moments $\tau_i=\E[g_i^2]$, $\omega_i=(\E[g_i])^2$, and $\kappa_i=\tau_i-\omega_i$ from Equation~\eqref{eq:gate_moments}; the audit verifies this decomposition to a maximum scaled error of $1.01\times10^{-7}$. On the held-out test split we report $\Delta\mathrm{KL}$ from the original to the intervened predictive distribution, raw gold-label $\Delta\mathrm{CE}$ and $\Delta\mathrm{Acc}$, variance-normalized activation error, and soft/hard cardinality. The NMSE variance normalizer is floored at $10^{-12}$. Numerically, ICA uses a reversible power-of-two rescaling and a float64 solve; only non-finite EuroBERT float16 gradients receive adaptive inverse loss scaling; and rare negative shifted-NMF calibration and test coordinates are clipped at the fitted nonnegative boundary.

\paragraph{Statistical scope and compute.}
There is one seed per model. Checkpoints, layers, methods, and objectives are structured repeated measurements, not IID replicates; analyses therefore use exact within-cell pairing and descriptive model-level medians or interquartile ranges rather than row-level significance tests. Because $\lambda=10^{-3}$ is shared but not normalized by $K_{\max}$ or objective scale, criterion comparisons describe a single sparsity--fidelity operating point, not an equal-budget or matched-fidelity frontier. Experiments ran on four NVIDIA A100-SXM4 40GB GPUs with an AMD EPYC 7742 CPU and approximately 503GiB RAM.

\paragraph{Software environment and computational record.}
The preserved cluster snapshot records Ubuntu 24.04.2, Python~\texttt{3.12.3}, PyTorch~\texttt{2.12.1+cu126}, Transformers~\texttt{5.12.1}, Datasets~\texttt{5.0.0}, Accelerate~\texttt{1.14.0}, NumPy~\texttt{2.5.0}, SciPy~\texttt{1.18.0}, scikit-learn~\texttt{1.9.0}, pandas~\texttt{3.0.3}, CUDA runtime~\texttt{12.6}, and cuDNN~\texttt{91002}. Source shuffling adds offsets \texttt{C4:0}, \texttt{WikiText-103:4000}, \texttt{FineWeb-Edu:7000}, \texttt{Wikipedia-en:11000}, and \texttt{Wikipedia-fr:13000} to base seed 13, yielding source seeds 13, 4013, 7013, 11013, and 13013. Exact accelerator-hour accounting was not retained; the experiments combined GPU-based activation/intervention computation with substantial CPU-based factorization and analysis.

\paragraph{Artifact licenses and intended use.}
The official release cards state Apache-2.0 for Pythia \citep{eleutherai2023pythia70m,eleutherai2023pythia160m,eleutherai2023pythia410m}, OLMo-1B \citep{allenai2024olmo1b0724}, and EuroBERT \citep{eurobert2025modelcard210m,eurobert2025modelcard610m}. The exact ModernCamemBERT checkpoint repository has no model card or license declaration \citep{almanach2025moderncamembertckpts}; its companion final-model card states MIT \citep{almanach2025moderncamembert}. We do not infer from the latter a separate license for the checkpoint repository. C4 and FineWeb-Edu are released under ODC-By 1.0 and remain subject to the applicable Common Crawl terms \citep{allenai2026c4card,lozhkov2024fineweb-edu}. The WikiText card is internally inconsistent: its repository metadata lists CC BY-SA 3.0 and GFDL, while its prose states CC BY-SA 4.0 \citep{salesforce2026wikitextcard}. The English and French Wikipedia configurations inherit the source-text terms stated on their card, CC BY-SA 3.0 and GFDL, with some text potentially subject only to CC BY-SA or additional terms \citep{wikimedia2023wikipedia}. Our use is limited to research analysis of pretrained checkpoints and sampled corpus text.

\FloatBarrier

\section{Additional Empirical Results}
\label{app:empirical_diagnostics}

This appendix resolves the main results by model, checkpoint, layer, decomposition, and metric. It first establishes the operating points and longitudinal profiles, then examines decomposition and selected-set composition, and finally reports geometry, threshold, and output-metric diagnostics.

\subsection{Operating points and model heterogeneity}

Table~\ref{tab:model-objective-ledger} expands the main dashboard into model-resolved summaries. In every model, downstream selection reaches a more compact soft-mask operating point than reconstruction; its occupancy ranges from $0.6\%$ for OLMo to $46.6\%$ for Pythia-70M. Each row is the median of 192 method--checkpoint--depth cells. Hard-mask empty rates range from zero to $31.2\%$, which motivates reporting continuous soft occupancy alongside thresholded composition. These are fixed-penalty operating points rather than matched-budget comparisons.

\begin{table*}[t]
\centering
\scriptsize
\setlength{\tabcolsep}{3.0pt}
\caption{Model-resolved behavioral and mask profiles for each monitoring objective.}
\label{tab:model-objective-ledger}
\resizebox{\textwidth}{!}{%
\begin{tabular}{@{}llrrrrrrrr@{}}
\toprule
Model & Objective & Cells & $K_{\mathrm{soft}}/K_{\max}$ (\%) & $K_{\mathrm{hard}}/K_{\max}$ (\%) & Rel. KL & Rel. CE & Acc. damage (pp) & NMSE & Empty (\%) \\
\midrule
Pythia-70M & Rec. & 192 & 84.2 & 100.0 & 0.032 & +0.021 & 1.13 & 0.167 & 0.0 \\
 & Lin. & 192 & 89.0 & 94.2 & 0.034 & +0.023 & 1.32 & 0.258 & 0.0 \\
 & Down. & 192 & 46.6 & 47.1 & 0.051 & +0.039 & 2.10 & 0.530 & 0.0 \\
 & Transfer & 192 & 43.5 & 43.0 & 0.169 & +0.117 & 5.37 & 0.886 & 12.5 \\
\midrule
Pythia-160M & Rec. & 192 & 75.2 & 99.5 & 0.016 & +0.012 & 0.70 & 0.215 & 0.0 \\
 & Lin. & 192 & 84.7 & 90.1 & 0.015 & +0.013 & 0.75 & 0.264 & 0.0 \\
 & Down. & 192 & 22.9 & 22.9 & 0.040 & +0.034 & 1.63 & 0.807 & 1.6 \\
 & Transfer & 192 & 37.1 & 32.5 & 0.099 & +0.108 & 4.73 & 1.182 & 12.0 \\
\midrule
Pythia-410M & Rec. & 192 & 64.9 & 90.9 & 0.009 & +0.006 & 0.32 & 0.299 & 0.0 \\
 & Lin. & 192 & 83.4 & 89.7 & 0.007 & +0.005 & 0.20 & 0.270 & 0.0 \\
 & Down. & 192 & 10.1 & 10.1 & 0.020 & +0.020 & 0.80 & 0.974 & 12.5 \\
 & Transfer & 192 & 28.6 & 21.1 & 0.032 & +0.035 & 0.90 & 0.953 & 15.6 \\
\midrule
OLMo-1B & Rec. & 192 & 20.8 & 12.9 & 0.028 & +0.025 & 0.77 & 0.573 & 0.0 \\
 & Lin. & 192 & 50.4 & 51.8 & 0.016 & +0.015 & 0.47 & 0.514 & 0.0 \\
 & Down. & 192 & 0.6 & 0.6 & 0.033 & +0.037 & 1.20 & 0.995 & 21.4 \\
 & Transfer & 192 & 4.2 & 2.1 & 0.093 & +0.083 & 2.68 & 1.006 & 10.4 \\
\midrule
EuroBERT-210M & Rec. & 192 & 73.1 & 85.2 & 0.001 & +0.001 & 0.03 & 0.144 & 0.0 \\
 & Lin. & 192 & 25.9 & 26.0 & 0.006 & +0.009 & 0.38 & 0.691 & 0.0 \\
 & Down. & 192 & 6.6 & 6.4 & 0.005 & +0.006 & 0.33 & 0.861 & 13.0 \\
 & Transfer & 192 & 33.7 & 32.2 & 0.006 & +0.011 & 0.38 & 0.948 & 0.0 \\
\midrule
EuroBERT-610M & Rec. & 192 & 57.5 & 63.2 & 0.006 & +0.003 & 0.37 & 0.277 & 0.0 \\
 & Lin. & 192 & 29.3 & 29.6 & 0.014 & +0.010 & 0.75 & 0.638 & 0.0 \\
 & Down. & 192 & 2.7 & 2.6 & 0.015 & +0.011 & 0.77 & 0.959 & 10.4 \\
 & Transfer & 192 & 11.6 & 10.5 & 0.020 & +0.017 & 1.07 & 1.009 & 4.7 \\
\midrule
ModernCamemBERT & Rec. & 192 & 73.0 & 90.4 & 0.002 & +0.001 & 0.17 & 0.148 & 0.0 \\
 & Lin. & 192 & 84.5 & 96.0 & 0.001 & +0.001 & 0.10 & 0.198 & 7.3 \\
 & Down. & 192 & 0.8 & 0.5 & 0.007 & +0.007 & 0.23 & 1.000 & 31.2 \\
 & Transfer & 192 & 23.7 & 19.4 & 0.112 & +0.124 & 3.72 & 1.785 & 12.0 \\
\bottomrule
\end{tabular}%
}
\vspace{2pt}
\begin{minipage}{0.98\linewidth}
\footnotesize\textit{Notes.} Values are within-model medians over four methods, eight scheduled checkpoints, and six tracked layers. Accuracy damage is positive when the intervention reduces target-token top-1 accuracy. Transfer is source-to-final and should not be compared as if it were a same-checkpoint objective. The 192 cells per row are structured repeated measurements, not IID replicates.
\end{minipage}
\end{table*}

Table~\ref{tab:pairwise-objectives} gives the exhaustive same-checkpoint comparison after exact matching. The dominance columns require both lower $\Delta\mathrm{KL}$ and no larger soft occupancy; $93.2\%$ of reconstruction--downstream pairs exhibit the expected sparsity--fidelity trade-off. Table~\ref{tab:transfer-summary} reports transfer separately and repeats the endpoint analysis after excluding learned self-alignment.

\begin{table*}[t]
\centering
\scriptsize
\setlength{\tabcolsep}{3.4pt}
\caption{Exact same-checkpoint contrasts among reconstruction, linear, and downstream monitoring objectives.}
\label{tab:pairwise-objectives}
\resizebox{\textwidth}{!}{%
\begin{tabular}{@{}lrrrrrrrr@{}}
\toprule
$A-B$ & Pairs & $\log_{10}(\Delta\mathrm{KL}_A/\Delta\mathrm{KL}_B)$ & $\Delta K_{\mathrm{soft}}/K_{\max}$ & $\Delta$ rel. CE & $\Delta$ acc. damage (pp) & $A$ dom. (\%) & $B$ dom. (\%) & Trade-off (\%) \\
\midrule
Rec. $-$ Lin. & 1,344 & $-0.000$ [$-0.158$, $+0.072$] & $-0.081$ [$-0.125$, $+0.136$] & $+0.000$ [$-0.003$, $+0.000$] & $-0.033$ [$-0.200$, $+0.033$] & 6.8 & 2.7 & 90.2 \\
Rec. $-$ Down. & 1,344 & $-0.404$ [$-0.525$, $-0.324$] & $+0.444$ [$+0.368$, $+0.474$] & $-0.007$ [$-0.013$, $-0.005$] & $-0.350$ [$-0.592$, $-0.283$] & 2.2 & 4.5 & 93.2 \\
Lin. $-$ Down. & 1,344 & $-0.354$ [$-0.432$, $-0.175$] & $+0.486$ [$+0.251$, $+0.599$] & $-0.013$ [$-0.014$, $-0.003$] & $-0.550$ [$-0.584$, $-0.150$] & 0.4 & 6.6 & 93.0 \\
\bottomrule
\end{tabular}%
}
\vspace{2pt}
\begin{minipage}{0.98\linewidth}
\footnotesize\textit{Notes.} Every pair is matched on model, method, checkpoint, and layer. A contrast is $A-B$: negative log-KL and accuracy-damage contrasts favor $A$, whereas negative occupancy makes $A$ sparser. Continuous effects are medians [model IQR] of model-level medians; rates are macro-averaged model rates. ``Trade-off'' combines cases in which the sparser mask has larger KL and cases in which the larger mask has smaller KL; the small remainder consists of boundary/tie cases (percentages may also differ by rounding). Transfer is excluded because it is not a same-checkpoint objective.
\end{minipage}
\end{table*}

\begin{table*}[t]
\centering
\begin{minipage}{0.96\textwidth}
\centering
\scriptsize
\setlength{\tabcolsep}{3.6pt}
\captionof{table}{Source-to-final transfer operating point and no-self endpoint change.}
\label{tab:transfer-summary}
\begin{tabular}{@{}lrrrrrrrr@{}}
\toprule
Scope & Models & Rel. KL & Rel. CE & Acc. damage (pp) & $K_s/K_{\max}$ (\%) & $K_h/K_{\max}$ (\%) & NMSE & No-self $\Delta\log_{10}\mathrm{KL}$ \\
\midrule
All eight sources & 7 & 0.093 & +0.083 & 2.68 & 28.6 & 21.1 & 1.006 & -- \\
Earliest$\rightarrow$penultimate & 7 & -- & -- & -- & -- & -- & -- & $-0.016$ \\
\bottomrule
\end{tabular}

\vspace{2pt}
\begin{minipage}{0.92\linewidth}
\footnotesize\textit{Notes.} Transfer masks use source coordinates evaluated in the final model. The final source is learned self-alignment, not a reconstruction baseline. The last column is the median model-level earliest-to-penultimate change, excluding that endpoint.
\end{minipage}
\end{minipage}

\vspace{8pt}

\begin{minipage}{0.96\textwidth}
\centering
\small
\setlength{\tabcolsep}{4pt}
\captionof{table}{Paired objective comparisons stratified by language-modeling objective.}
\label{tab:architecture_sensitivity}
\begin{tabular}{@{}llrrrr@{}}
\toprule
Group & Comparison & $\Delta\Delta$KL & Lower KL & $\Delta K_s/K_{\max}$ & Comparison dom. \\
\midrule
Causal LMs & Linear $-$ rec. & $-.00075$ & 59.0\% & $+.121$ & 4.0\% \\
Causal LMs & Downstream $-$ rec. & $+.0443$ & 11.7\% & $-.368$ & 6.1\% \\
Masked LMs & Linear $-$ rec. & $+.0225$ & 22.6\% & $-.269$ & .9\% \\
Masked LMs & Downstream $-$ rec. & $+.0247$ & 2.4\% & $-.483$ & 2.4\% \\
\bottomrule
\end{tabular}

\vspace{2pt}
\begin{minipage}{0.82\linewidth}
\footnotesize\textit{Notes.} $K_s=K_{\mathrm{soft}}$. Rates are macro-averaged over models and remain descriptive because there is one run per model and the groups differ in more than attention direction.
\end{minipage}
\end{minipage}
\end{table*}

Figure~\ref{fig:criterion_tradeoffs} shows that the compact downstream operating point holds in every model, with a corresponding increase in $\Delta\mathrm{KL}$ under the shared penalty; fixed-head selection is more heterogeneous. The visually separated transfer panel provides the source-to-final context and has a different estimand from the same-checkpoint contrasts.

\begin{figure*}[t]
\centering
\includegraphics[width=.91\textwidth]{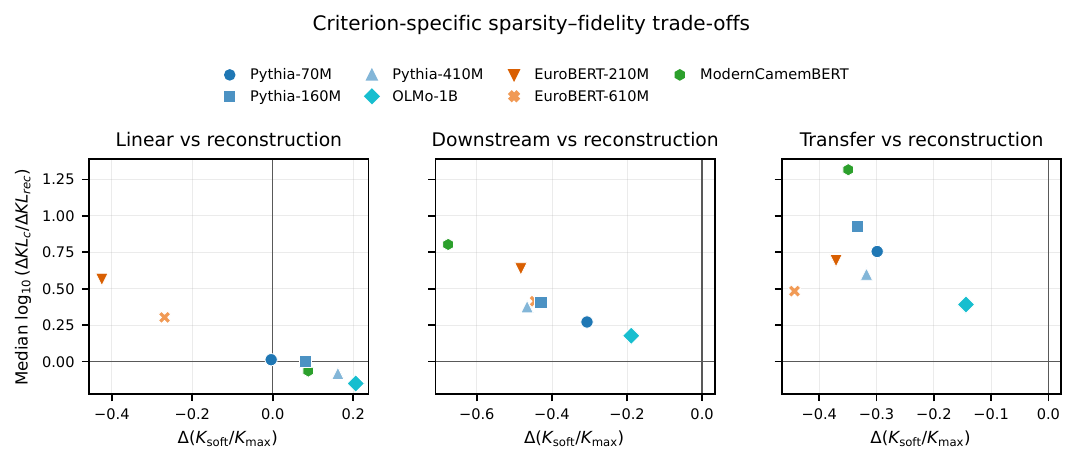}
\caption{\textbf{Target-specific sparsity--fidelity contrasts.} Same-checkpoint panels pair each model with reconstruction at the same method, checkpoint, and layer. The horizontal coordinate is the change in normalized soft cardinality; the vertical coordinate is the median log-ratio of held-out distribution shift. The separated transfer panel pairs source coordinates evaluated in the final model with reconstruction at the source and is contextual rather than an equal-estimand target comparison.}
\label{fig:criterion_tradeoffs}
\end{figure*}

Table~\ref{tab:architecture_sensitivity} further stratifies the paired comparison by language-modeling objective. Fixed-head selection has a lower-KL rate of $59.0\%$ in the four autoregressive models and $22.6\%$ in the three masked models; the corresponding downstream rates are $11.7\%$ and $2.4\%$. This stable group contrast motivates retaining model identity, although the two groups also differ in data, tokenizer, and training recipe.

\FloatBarrier

\subsection{Checkpoint, depth, and scale}

Figure~\ref{fig:training_dynamics} shows normalized checkpoint trajectories before reduction to correlations. For reconstruction, fixed-head, and downstream selection, $\Delta\mathrm{KL}$ ends above its initial value in all three Pythia models, OLMo, and ModernCamemBERT, and below it in EuroBERT-210M. These coherent within-model profiles supply the intuition behind Figure~\ref{fig:architecture_fingerprints}; EuroBERT-610M and aggregate transfer show more mixed paths, with transfer resolving cleanly by method below.

\begin{figure*}[t]
\centering
\includegraphics[width=.93\textwidth]{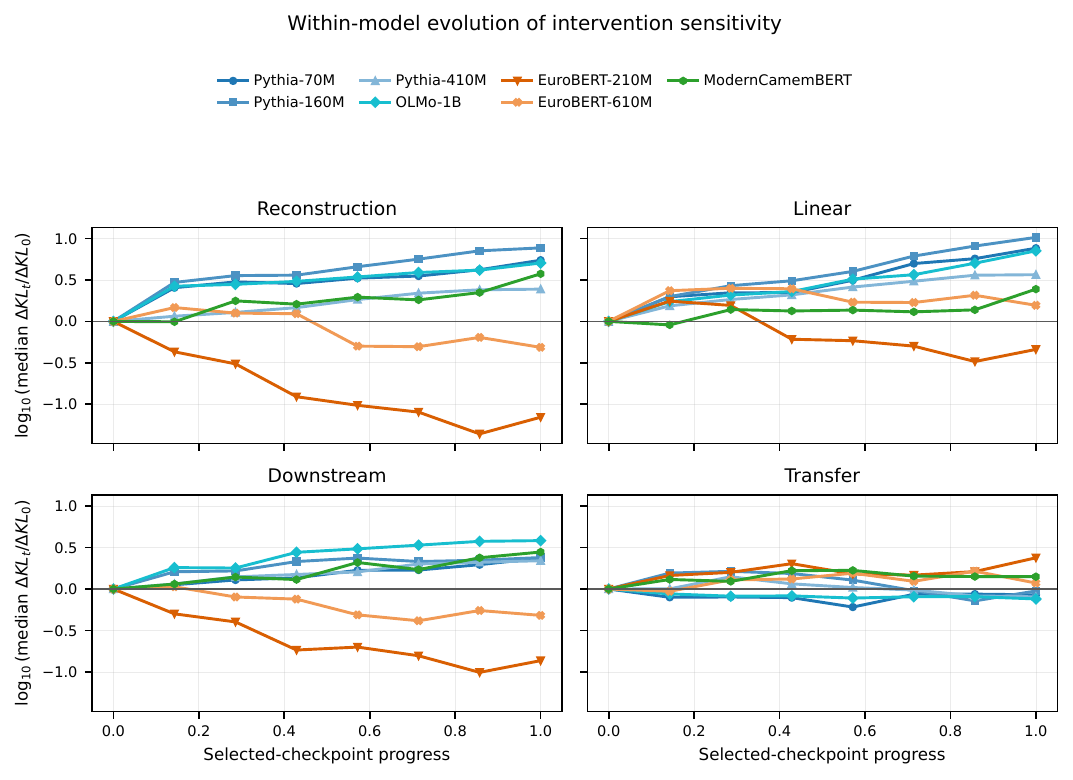}
\caption{\textbf{Within-model evolution of selected-mask intervention effects.} At each checkpoint, $\Delta\mathrm{KL}$ is summarized over decomposition methods and tracked layers and divided by the initial-checkpoint median. The horizontal axis is scheduled-checkpoint rank; connected points are not independent replicates.}
\label{fig:training_dynamics}
\end{figure*}

The layer-resolved view in Figure~\ref{fig:layerwise_temporal} shows positive checkpoint-rank associations at every sampled layer of each Pythia model and ModernCamemBERT for the three same-checkpoint objectives, so the temporal signature is network-wide rather than confined to a late block. EuroBERT-210M supplies the coherent opposite profile, with negative associations at every sampled layer for reconstruction and downstream selection; EuroBERT-610M weakens this pattern at its final sampled layer. Transfer contains opposing early- and late-layer tendencies and ends in learned self-alignment.

\begin{figure*}[t]
\centering
\includegraphics[width=.98\textwidth]{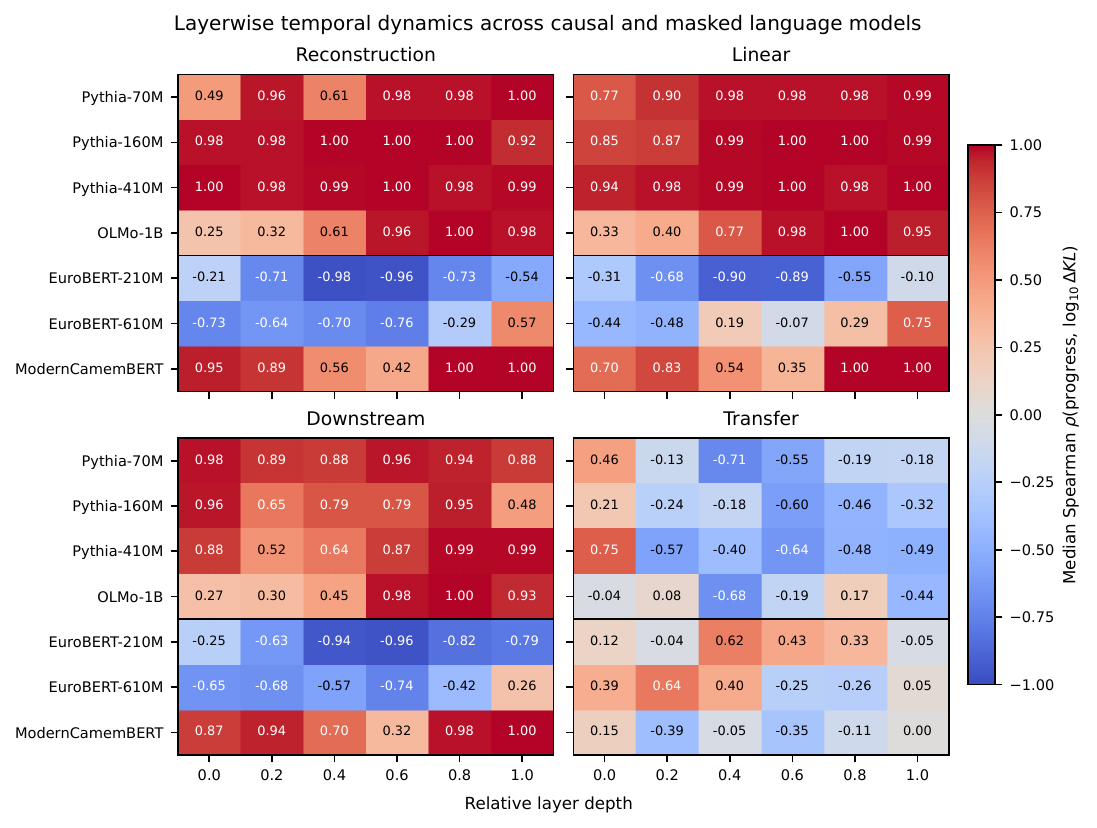}
\caption{\textbf{Layer-resolved checkpoint dynamics.} Each cell is the median across decomposition methods of the within-layer Spearman correlation between the eight common-grid checkpoint ranks and $\log_{10}\Delta\mathrm{KL}$. Positive values indicate a larger selected-mask effect at later scheduled checkpoints; they do not imply a monotone path. Transfer includes learned self-alignment as its final point and cannot be read as a continuity trend.}
\label{fig:layerwise_temporal}
\end{figure*}

The Pythia size comparison gives a strong matched-cell pattern (Figure~\ref{fig:architecture_scale}): Pythia-410M has lower $\Delta\mathrm{KL}$ than Pythia-70M in $92.8\%$ of cells, a smaller normalized mask in $78.8\%$, and both properties in $72.4\%$. EuroBERT-610M instead has a smaller mask in $65.4\%$ of matched cells but lower $\Delta\mathrm{KL}$ in $20.7\%$, showing that the relationship between size, sparsity, and fidelity is family-dependent rather than a universal scaling law.

\begin{figure*}[t]
\centering
\includegraphics[width=.88\textwidth]{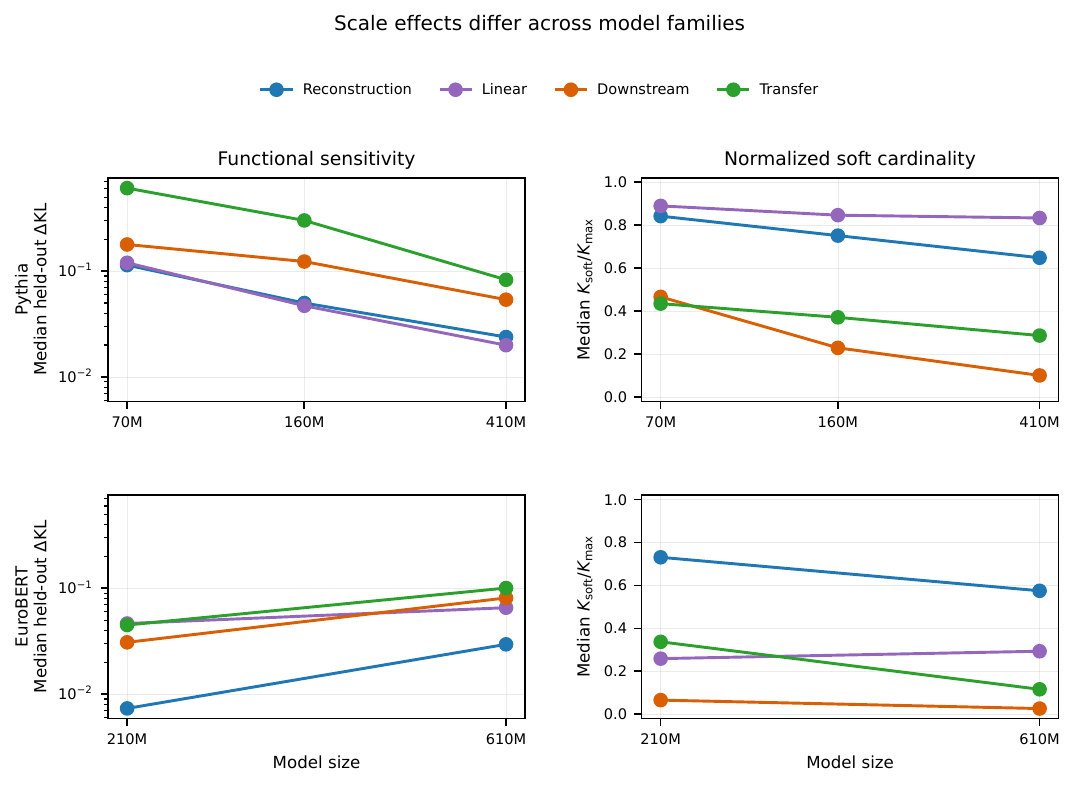}
\caption{\textbf{Within-family size contrasts.} Points are model-level medians over the matched method, target, checkpoint-rank, and depth-rank grid. $K_{\max}$ is clamp-determined and the sparsity penalty is not width-normalized.}
\label{fig:architecture_scale}
\end{figure*}

Figure~\ref{fig:depth_relocation} asks whether the location and breadth of a behavioral effect change between the first and final scheduled checkpoints. For each checkpoint, the nonnegative effect magnitudes across the six sampled layers are normalized to sum to one. The center is their mean sampled-depth rank, and support is their entropy effective count divided by six. The $\Delta\mathrm{KL}$ center moves later for reconstruction in six of seven models and for fixed-head selection in all seven; downstream relocation and support remain model-dependent (Table~\ref{tab:temporal-depth-ledger}). Because normalization removes total magnitude, center and support describe relative allocation across the sampled layers.

\begin{figure*}[t]
\centering
\includegraphics[width=.98\textwidth]{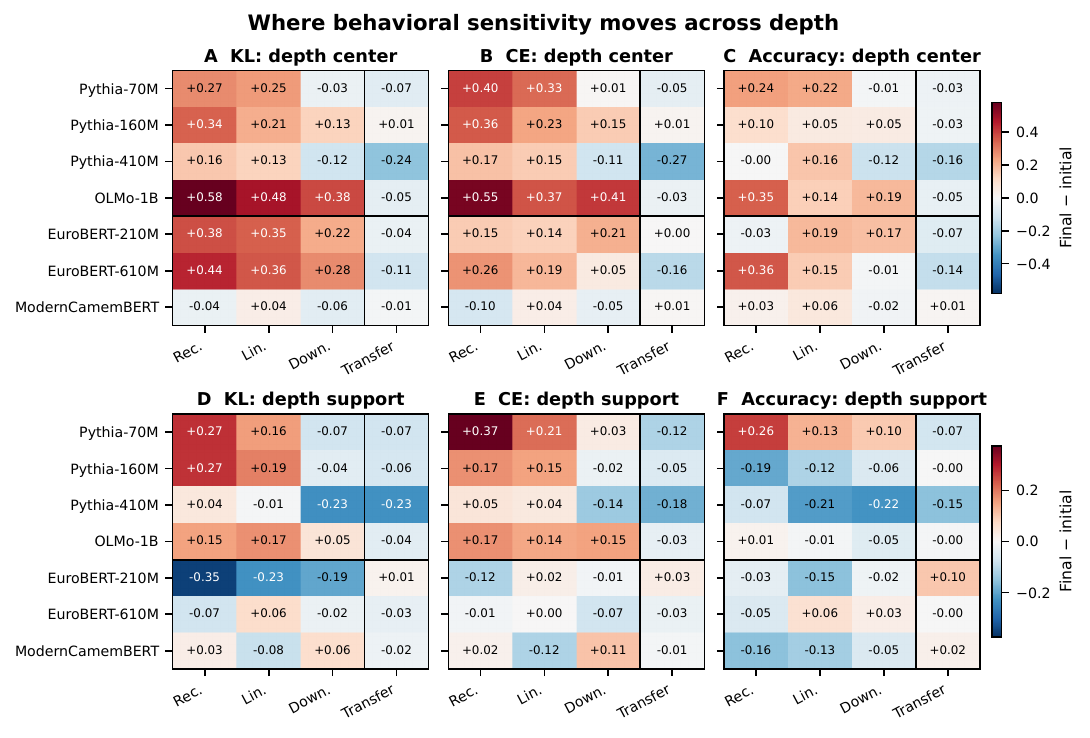}
\caption{\textbf{Endpoint relocation across sampled depth ranks.} A--C report final-minus-initial changes in the effect-weighted center for $\Delta\mathrm{KL}$, absolute relative $\Delta\mathrm{CE}$, and absolute $\Delta\mathrm{Acc}$. D--F report changes in normalized entropy support. Positive center values indicate movement toward later among the six sampled MLP outputs; positive support values indicate a more broadly distributed effect. Cells are medians over four decompositions. Transfer includes learned self-alignment and cannot be interpreted as a continuity trend.}
\label{fig:depth_relocation}
\end{figure*}

Endpoint displacement does not determine the path between endpoints. For a sequence $x_1,\ldots,x_8$, Figure~\ref{fig:path_shape} reports $|x_8-x_1|/\sum_{j=1}^{7}|x_{j+1}-x_j|$, with undefined flat paths omitted. A value near one indicates movement largely in one direction on the scheduled grid; a value near zero indicates reversals or a small net change. Across the 21 same-checkpoint model--target summaries, its median is $0.72$ for $\log_{10}\Delta\mathrm{KL}$, $0.56$ for soft occupancy, and $0.46$ for $\log_{10}$ NMSE. This is a path-directness diagnostic over eight checkpoint ranks, not a claim about smoothness between checkpoints. Table~\ref{tab:temporal-depth-ledger} gives the numerical summaries; transfer is omitted because its self-alignment endpoint changes the comparison.

\begin{figure*}[t]
\centering
\includegraphics[width=.98\textwidth]{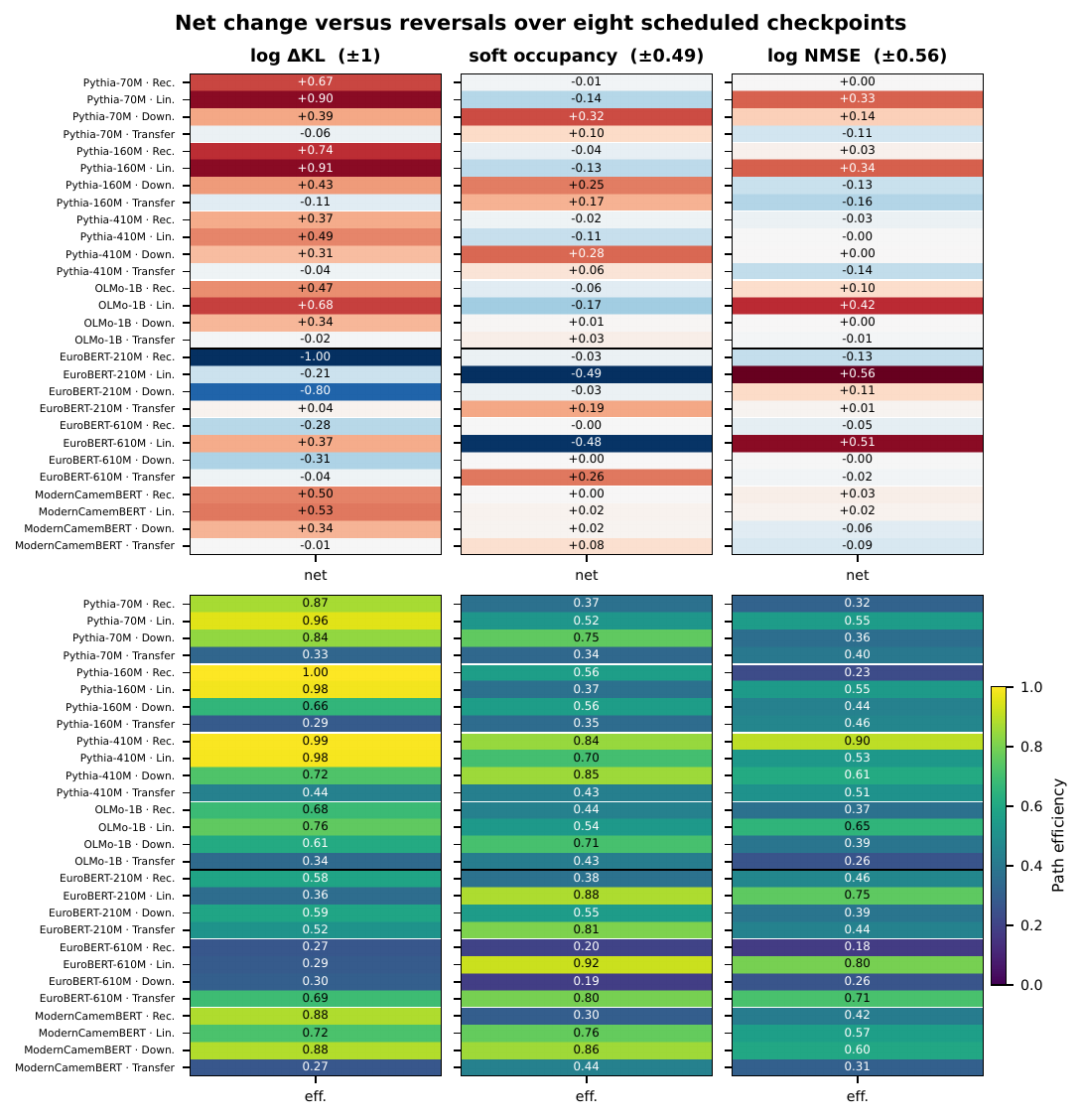}
\caption{\textbf{Net change and directness on the scheduled checkpoint grid.} Columns show $\log_{10}\Delta\mathrm{KL}$, normalized soft occupancy, and $\log_{10}$ NMSE. The upper row reports endpoint change and the lower row reports the ratio of endpoint displacement to total adjacent variation, summarized within model and target over method--layer strata. Transfer includes its learned self-alignment endpoint and should not be compared directly with the three same-checkpoint targets.}
\label{fig:path_shape}
\end{figure*}

\begin{table*}[t]
\centering
\scriptsize
\setlength{\tabcolsep}{3.2pt}
\caption{Same-checkpoint temporal paths and relocation of selected-mask effects across the tracked layers.}
\label{tab:temporal-depth-ledger}
\resizebox{\textwidth}{!}{%
\begin{tabular}{@{}llrrrrrrrr@{}}
\toprule
Model & Objective & $\Delta\log_{10}\mathrm{KL}$ & Eff. KL & $\Delta K_s/K_{\max}$ & Eff. $K_s$ & $\Delta\log_{10}\mathrm{NMSE}$ & Eff. NMSE & $\Delta$ KL center & $\Delta$ KL support \\
\midrule
Pythia-70M & Rec. & +0.67 & 0.87 & $-0.01$ & 0.37 & +0.00 & 0.32 & +0.27 & +0.27 \\
 & Lin. & +0.90 & 0.96 & $-0.14$ & 0.52 & +0.33 & 0.55 & +0.25 & +0.16 \\
 & Down. & +0.39 & 0.84 & +0.32 & 0.75 & +0.14 & 0.36 & $-0.03$ & $-0.07$ \\
\midrule
Pythia-160M & Rec. & +0.74 & 1.00 & $-0.04$ & 0.56 & +0.03 & 0.23 & +0.34 & +0.27 \\
 & Lin. & +0.91 & 0.98 & $-0.13$ & 0.37 & +0.34 & 0.55 & +0.21 & +0.19 \\
 & Down. & +0.43 & 0.66 & +0.25 & 0.56 & $-0.13$ & 0.44 & +0.13 & $-0.04$ \\
\midrule
Pythia-410M & Rec. & +0.37 & 0.99 & $-0.02$ & 0.84 & $-0.03$ & 0.90 & +0.16 & +0.04 \\
 & Lin. & +0.49 & 0.98 & $-0.11$ & 0.70 & $-0.00$ & 0.53 & +0.13 & $-0.01$ \\
 & Down. & +0.31 & 0.72 & +0.28 & 0.85 & +0.00 & 0.61 & $-0.12$ & $-0.23$ \\
\midrule
OLMo-1B & Rec. & +0.47 & 0.68 & $-0.06$ & 0.44 & +0.10 & 0.37 & +0.58 & +0.15 \\
 & Lin. & +0.68 & 0.76 & $-0.17$ & 0.54 & +0.42 & 0.65 & +0.48 & +0.17 \\
 & Down. & +0.34 & 0.61 & +0.01 & 0.71 & +0.00 & 0.39 & +0.38 & +0.05 \\
\midrule
EuroBERT-210M & Rec. & $-1.00$ & 0.58 & $-0.03$ & 0.38 & $-0.13$ & 0.46 & +0.38 & $-0.35$ \\
 & Lin. & $-0.21$ & 0.36 & $-0.49$ & 0.88 & +0.56 & 0.75 & +0.35 & $-0.23$ \\
 & Down. & $-0.80$ & 0.59 & $-0.03$ & 0.55 & +0.11 & 0.39 & +0.22 & $-0.19$ \\
\midrule
EuroBERT-610M & Rec. & $-0.28$ & 0.27 & $-0.00$ & 0.20 & $-0.05$ & 0.18 & +0.44 & $-0.07$ \\
 & Lin. & +0.37 & 0.29 & $-0.48$ & 0.92 & +0.51 & 0.80 & +0.36 & +0.06 \\
 & Down. & $-0.31$ & 0.30 & +0.00 & 0.19 & $-0.00$ & 0.26 & +0.28 & $-0.02$ \\
\midrule
ModernCamemBERT & Rec. & +0.50 & 0.88 & +0.00 & 0.30 & +0.03 & 0.42 & $-0.04$ & +0.03 \\
 & Lin. & +0.53 & 0.72 & +0.02 & 0.76 & +0.02 & 0.57 & +0.04 & $-0.08$ \\
 & Down. & +0.34 & 0.88 & +0.02 & 0.86 & $-0.06$ & 0.60 & $-0.06$ & +0.06 \\
\bottomrule
\end{tabular}%
}
\vspace{2pt}
\begin{minipage}{0.98\linewidth}
\footnotesize\textit{Notes.} Transfer rows are excluded because their final source is a self-alignment endpoint. For a checkpoint sequence $x_1,\ldots,x_8$, path efficiency is $|x_8-x_1|/\sum_{j=1}^{7}|x_{j+1}-x_j|$; values near one indicate a nearly monotone path, while low values indicate reversals or a small net endpoint change. Entries are medians over method$\times$layer strata. The eight checkpoints are scheduled ordinal positions, not a dense or equally spaced time series. The depth center uses coordinates $j/5$ for the ordered six tracked layers, and thus measures sampled-layer rank rather than the exact normalized layer index; support is the entropy effective support of the six nonnegative KL shares, divided by six.
\end{minipage}
\end{table*}

\FloatBarrier

\subsection{Decomposition and checkpoint transfer}

\begin{table*}[t]
\centering
\scriptsize
\setlength{\tabcolsep}{2.8pt}
\caption{Model- and decomposition-resolved intervention behavior for all four monitoring objectives.}
\label{tab:method-behavior}
\resizebox{\textwidth}{!}{%
\begin{tabular}{@{}llrrrrrrrr@{}}
\toprule
Model & Method & \multicolumn{2}{c}{Rec.} & \multicolumn{2}{c}{Lin.} & \multicolumn{2}{c}{Down.} & \multicolumn{2}{c}{Transfer} \\
\cmidrule(lr){3-4}\cmidrule(lr){5-6}\cmidrule(lr){7-8}\cmidrule(l){9-10}
 & & Rel. KL & $K_s/K_{\max}$ (\%) & Rel. KL & $K_s/K_{\max}$ (\%) & Rel. KL & $K_s/K_{\max}$ (\%) & Rel. KL & $K_s/K_{\max}$ (\%) \\
\midrule
Pythia-70M & SVD & 0.024 & 61.2 & 0.018 & 70.5 & 0.038 & 44.3 & 0.160 & 31.2 \\
 & ICA & 0.018 & 83.3 & 0.016 & 91.7 & 0.036 & 58.3 & 0.106 & 58.4 \\
 & NMF & 0.077 & 100.0 & 0.091 & 95.7 & 0.115 & 36.1 & 0.948 & 0.2 \\
 & SemiNMF & 0.033 & 92.0 & 0.047 & 84.8 & 0.062 & 52.4 & 0.127 & 62.5 \\
\midrule
Pythia-160M & SVD & 0.013 & 50.1 & 0.010 & 68.2 & 0.031 & 17.3 & 0.095 & 21.3 \\
 & ICA & 0.009 & 72.4 & 0.009 & 86.6 & 0.030 & 30.4 & 0.102 & 49.6 \\
 & NMF & 0.032 & 99.7 & 0.041 & 94.5 & 0.071 & 11.6 & 1.143 & 6.4 \\
 & SemiNMF & 0.016 & 78.7 & 0.025 & 88.5 & 0.041 & 30.6 & 0.053 & 44.7 \\
\midrule
Pythia-410M & SVD & 0.006 & 47.5 & 0.005 & 62.7 & 0.016 & 8.7 & 0.045 & 20.8 \\
 & ICA & 0.005 & 65.3 & 0.004 & 82.4 & 0.017 & 16.9 & 0.027 & 46.6 \\
 & NMF & 0.012 & 100.0 & 0.013 & 91.4 & 0.031 & 5.3 & 0.035 & 0.0 \\
 & SemiNMF & 0.009 & 64.4 & 0.008 & 88.0 & 0.021 & 14.6 & 0.026 & 34.2 \\
\midrule
OLMo-1B & SVD & 0.014 & 14.0 & 0.009 & 29.5 & 0.021 & 0.5 & 0.070 & 2.9 \\
 & ICA & 0.017 & 24.7 & 0.010 & 51.6 & 0.026 & 0.2 & 0.048 & 1.8 \\
 & NMF & 0.145 & 86.2 & 0.189 & 54.1 & 0.060 & 0.9 & 0.382 & 10.0 \\
 & SemiNMF & 0.023 & 15.6 & 0.013 & 65.1 & 0.036 & 0.3 & 0.063 & 5.5 \\
\midrule
EuroBERT-210M & SVD & 0.001 & 42.0 & 0.003 & 4.3 & 0.005 & 1.4 & 0.004 & 15.8 \\
 & ICA & 0.001 & 73.2 & 0.003 & 15.5 & 0.006 & 0.2 & 0.002 & 26.5 \\
 & NMF & 0.001 & 100.0 & 0.117 & 48.4 & 0.002 & 96.3 & 0.418 & 44.9 \\
 & SemiNMF & 0.002 & 73.4 & 0.005 & 29.3 & 0.007 & 8.2 & 0.004 & 39.1 \\
\midrule
EuroBERT-610M & SVD & 0.005 & 30.3 & 0.009 & 5.6 & 0.011 & 1.2 & 0.015 & 6.3 \\
 & ICA & 0.004 & 58.8 & 0.010 & 22.7 & 0.014 & 0.3 & 0.008 & 8.6 \\
 & NMF & 0.006 & 99.1 & 0.096 & 51.5 & 0.018 & 78.4 & 0.702 & 41.5 \\
 & SemiNMF & 0.006 & 57.0 & 0.012 & 36.2 & 0.017 & 3.5 & 0.017 & 9.0 \\
\midrule
ModernCamemBERT & SVD & 0.001 & 44.5 & 0.000 & 41.4 & 0.005 & 0.3 & 0.069 & 9.2 \\
 & ICA & 0.001 & 72.7 & 0.000 & 92.7 & 0.006 & 0.0 & 0.041 & 29.7 \\
 & NMF & 0.004 & 100.0 & 0.004 & 98.0 & 0.009 & 1.5 & 0.460 & 34.2 \\
 & SemiNMF & 0.001 & 76.9 & 0.001 & 60.8 & 0.015 & 0.7 & 0.060 & 31.8 \\
\bottomrule
\end{tabular}%
}
\vspace{2pt}
\begin{minipage}{0.98\linewidth}
\footnotesize\textit{Notes.} Each entry is the within-model/method median over the eight-checkpoint common grid and six tracked layers. $K_s=K_{\mathrm{soft}}$. Transfer is source-to-final. The table exposes decomposition-specific differences hidden by the model-level dashboard; rows are not independent replications.
\end{minipage}
\end{table*}

The model-resolved stratification in Table~\ref{tab:method-behavior} distinguishes raw fidelity from the paired Pareto rates in Figure~\ref{fig:method_pareto}. ICA often attains low raw $\Delta\mathrm{KL}$ with larger masks, whereas SVD has the highest within-cell Pareto-dominance rate for each of the four objectives. The shared candidate budget makes this a direct comparison of the executed coordinate systems, including their factorization constraints, preprocessing, and optimization behavior.

\begin{figure*}[t]
\centering
\includegraphics[width=.56\textwidth]{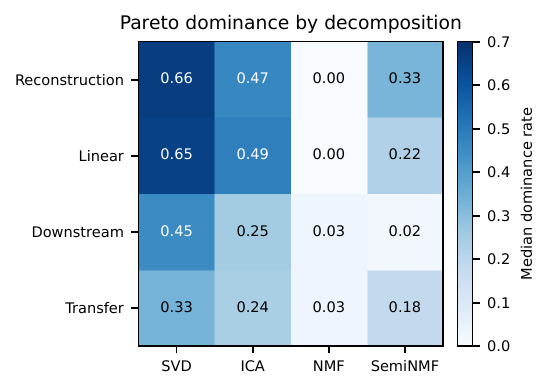}
\caption{\textbf{Pareto dominance by decomposition.} A method dominates a competitor in a matched cell when it has lower held-out $\Delta\mathrm{KL}$ and no larger $K_{\mathrm{soft}}/K_{\max}$. Entries are medians across the seven model-level dominance rates.}
\label{fig:method_pareto}
\end{figure*}

Figure~\ref{fig:transfer_methods} reports the decomposition-specific endpoint changes directly. From the earliest source to learned self-alignment, held-out transfer $\Delta\mathrm{KL}$ decreases for SVD and ICA in all seven models, compared with two models for NMF and one for SemiNMF. This repeated split explains the nearly stable aggregate transfer curve and shows that the monitor distinguishes how alignment interacts with each coordinate family; the endpoint comparison does not require a monotone intervening path.

\begin{figure*}[t]
\centering
\includegraphics[width=.71\textwidth]{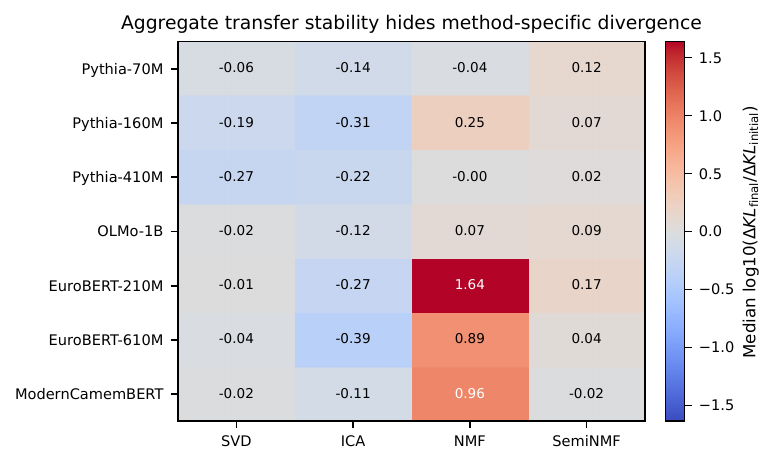}
\caption{\textbf{Transfer endpoint change by decomposition.} Cells give the median paired $\log_{10}(\Delta\mathrm{KL}_{\rm self}/\Delta\mathrm{KL}_{\rm earliest\mbox{-}source})$ across tracked layers. The numerator is a learned self-alignment endpoint; negative values therefore indicate lower held-out transfer KL there, not monotonic improvement in cross-checkpoint transfer.}
\label{fig:transfer_methods}
\end{figure*}

\FloatBarrier

\subsection{Geometry and supervised gate gradients}

Figure~\ref{fig:selection_enrichment} disaggregates the selection comparison by model. The denominator $K_{\mathrm{hard}}/K_{\max}$ is the analytic expected mass share under uniform same-cardinality selection, so $1\times$ is the size-only baseline; no random subsets are sampled. Empty hard masks are undefined for this statistic and excluded.

\begin{figure*}[t]
\centering
\includegraphics[width=.91\textwidth]{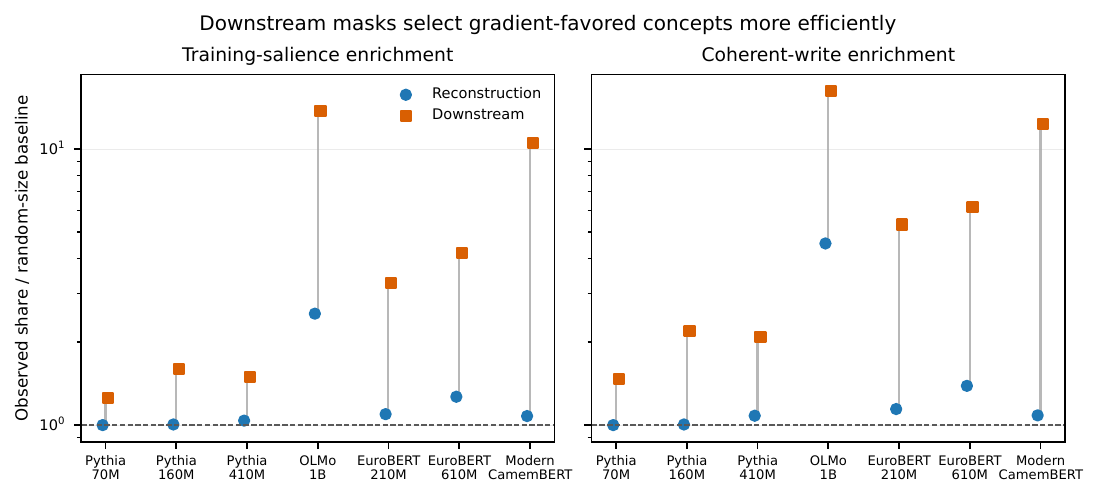}
\caption{\textbf{Gradient-score enrichment of nonempty hard masks.} Enrichment is the selected share of $\tau$ or $\omega$ divided by $K_{\mathrm{hard}}/K_{\max}$. Points are within-model medians, and segments pair reconstruction and downstream selection. Here $\omega$ is the squared mean signed gate gradient, not a realized parameter update.}
\label{fig:selection_enrichment}
\end{figure*}

Table~\ref{tab:selected-set-ledger} reports the corresponding counts and aggregates for every target. Energy and $\omega$ overlap compare each nonempty hard selection with the equal-size top-ranked reference set; mass and enrichment use the saved additive scores. Only genuine thresholded selections enter these summaries; empty sets and their operational argmax fallback are excluded. Because concept identities were not retained, these statistics measure within-cell composition rather than cross-checkpoint persistence.

\begin{table*}[t]
\centering
\scriptsize
\setlength{\tabcolsep}{2.8pt}
\caption{Saved score aggregates of genuine nonempty hard selections.}
\label{tab:selected-set-ledger}
\resizebox{\textwidth}{!}{%
\begin{tabular}{@{}llrrrrrrrrr@{}}
\toprule
Model & Objective & Valid & Excl. & $K_h/K_{\max}$ (\%) & Energy overlap & $\omega$ overlap & $\tau$ mass & $\omega$ mass & $\tau$ enrich. & $\omega$ enrich. \\
\midrule
Pythia-70M & Rec. & 192 & 0 & 100.0 & 1.00 & 1.00 & 1.00 & 1.00 & 1.00$\times$ & 1.00$\times$ \\
 & Lin. & 192 & 0 & 94.2 & 0.95 & 0.94 & 0.99 & 1.00 & 1.03$\times$ & 1.04$\times$ \\
 & Down. & 192 & 0 & 47.1 & 0.65 & 0.58 & 0.74 & 0.93 & 1.26$\times$ & 1.47$\times$ \\
 & Transfer & 168 & 24 & 51.2 & 0.72 & 0.57 & 0.53 & 0.51 & 1.05$\times$ & 1.05$\times$ \\
\midrule
Pythia-160M & Rec. & 192 & 0 & 99.5 & 1.00 & 1.00 & 1.00 & 1.00 & 1.01$\times$ & 1.01$\times$ \\
 & Lin. & 192 & 0 & 90.1 & 0.94 & 0.91 & 0.98 & 1.00 & 1.04$\times$ & 1.06$\times$ \\
 & Down. & 189 & 3 & 22.9 & 0.46 & 0.42 & 0.44 & 0.75 & 1.60$\times$ & 2.19$\times$ \\
 & Transfer & 169 & 23 & 39.2 & 0.59 & 0.47 & 0.41 & 0.46 & 1.18$\times$ & 1.22$\times$ \\
\midrule
Pythia-410M & Rec. & 192 & 0 & 90.9 & 0.98 & 0.92 & 0.95 & 1.00 & 1.04$\times$ & 1.08$\times$ \\
 & Lin. & 192 & 0 & 89.7 & 0.92 & 0.90 & 0.97 & 1.00 & 1.06$\times$ & 1.09$\times$ \\
 & Down. & 168 & 24 & 14.1 & 0.43 & 0.41 & 0.40 & 0.64 & 1.50$\times$ & 2.08$\times$ \\
 & Transfer & 162 & 30 & 27.2 & 0.62 & 0.45 & 0.36 & 0.35 & 1.28$\times$ & 1.29$\times$ \\
\midrule
OLMo-1B & Rec. & 192 & 0 & 12.9 & 0.95 & 0.50 & 0.67 & 0.95 & 2.53$\times$ & 4.55$\times$ \\
 & Lin. & 192 & 0 & 51.8 & 0.79 & 0.63 & 0.79 & 0.96 & 1.36$\times$ & 1.68$\times$ \\
 & Down. & 151 & 41 & 1.0 & 0.32 & 0.25 & 0.20 & 0.30 & 13.70$\times$ & 16.20$\times$ \\
 & Transfer & 172 & 20 & 2.5 & 0.33 & 0.23 & 0.08 & 0.10 & 2.67$\times$ & 2.86$\times$ \\
\midrule
EuroBERT-210M & Rec. & 192 & 0 & 85.2 & 0.99 & 0.87 & 0.98 & 1.00 & 1.10$\times$ & 1.14$\times$ \\
 & Lin. & 192 & 0 & 26.0 & 0.75 & 0.59 & 0.69 & 0.84 & 2.06$\times$ & 2.90$\times$ \\
 & Down. & 167 & 25 & 7.3 & 0.85 & 0.67 & 0.71 & 0.97 & 3.26$\times$ & 5.32$\times$ \\
 & Transfer & 192 & 0 & 32.2 & 0.70 & 0.56 & 0.46 & 0.59 & 1.48$\times$ & 1.63$\times$ \\
\midrule
EuroBERT-610M & Rec. & 192 & 0 & 63.2 & 0.97 & 0.74 & 0.90 & 0.97 & 1.27$\times$ & 1.39$\times$ \\
 & Lin. & 192 & 0 & 29.6 & 0.79 & 0.57 & 0.71 & 0.86 & 1.87$\times$ & 2.26$\times$ \\
 & Down. & 172 & 20 & 3.2 & 0.71 & 0.54 & 0.50 & 0.82 & 4.19$\times$ & 6.16$\times$ \\
 & Transfer & 183 & 9 & 12.7 & 0.53 & 0.44 & 0.18 & 0.35 & 1.60$\times$ & 1.76$\times$ \\
\midrule
ModernCamemBERT & Rec. & 192 & 0 & 90.4 & 0.99 & 0.96 & 0.99 & 1.00 & 1.08$\times$ & 1.08$\times$ \\
 & Lin. & 178 & 14 & 97.4 & 0.98 & 0.98 & 1.00 & 1.00 & 1.02$\times$ & 1.02$\times$ \\
 & Down. & 132 & 60 & 3.1 & 0.75 & 0.55 & 0.49 & 0.70 & 10.52$\times$ & 12.31$\times$ \\
 & Transfer & 169 & 23 & 23.9 & 0.46 & 0.40 & 0.35 & 0.43 & 1.42$\times$ & 1.50$\times$ \\
\bottomrule
\end{tabular}%
}
\vspace{2pt}
\begin{minipage}{0.98\linewidth}
\footnotesize\textit{Notes.} ``Valid'' counts genuine nonempty hard selections used for set summaries; ``Excl.'' counts cells with $K_h=0$, whose operational argmax fallback is not treated as a selected set. Every non-count entry is an independently computed within-model median over valid method$\times$checkpoint$\times$layer cells, and $K_h/K_{\max}$ is therefore conditional on a nonempty hard set. Energy and $\omega$ overlaps compare the selected set with the corresponding top-$K_h$ ranking. Enrichment is the median of the cellwise ratio $\text{mass}/(K_h/K_{\max})$, not the ratio of displayed medians. Here $\tau$ is the supervised-gradient second moment and $\omega=(\overline{g_i})^2$ is the squared mean-gradient moment. Selected identities were not saved, so the table does not measure identity persistence.
\end{minipage}
\end{table*}

Across all seven models, downstream hard selections are more enriched than reconstruction selections in both supervised gate-gradient moments. Among nonempty cells, their model-level median hard occupancy is only $1.0$--$47.1\%$, yet they retain $20$--$74\%$ of $\tau$ and $30$--$97\%$ of $\omega$. The resulting enrichment spans $1.26$--$13.70\times$ for $\tau$ and $1.47$--$16.20\times$ for $\omega$. This is a consistent cross-model composition result: compact output-preserving masks concentrate far more supervised sensitivity than their cardinality alone predicts.

Figure~\ref{fig:geometry_gradient} separates three associations. Energy and $\tau$ remain strongly rank-correlated across scheduled checkpoints, while the lower $\tau$--$\omega$ association shows that example-level gradient signs can cancel even when squared sensitivity is large. Reconstruction NMSE and held-out $\Delta\mathrm{KL}$ are weakly positively associated for reconstruction and fixed-head selection but negatively associated for downstream selection after within-model--method aggregation, reinforcing that geometric and behavioral preservation order cells differently.

\begin{figure*}[t]
\centering
\includegraphics[width=.98\textwidth]{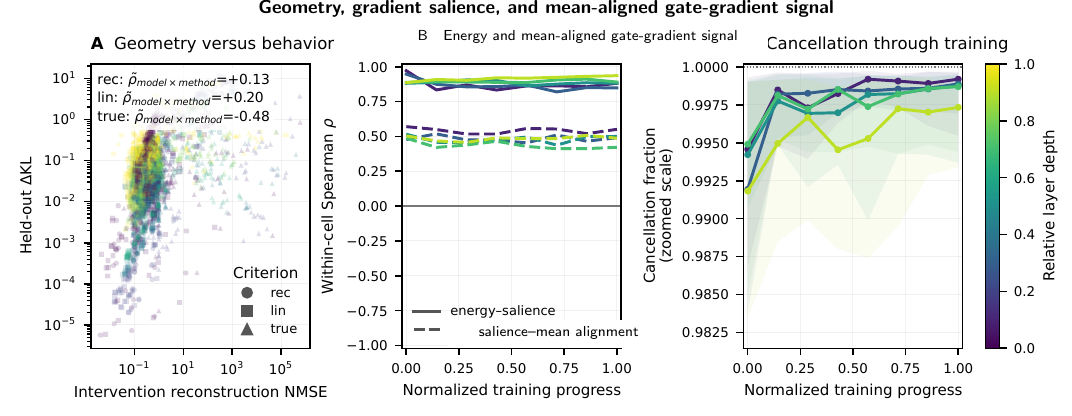}
\caption{\textbf{Geometry, gradient salience, and mean signed gate-gradient signal.} A: reconstruction NMSE against held-out $\Delta\mathrm{KL}$. B: within-cell rank associations of energy with $\tau$ and of $\tau$ with $\omega$. C: $1-\sum_i\omega_i/\sum_i\tau_i$ across scheduled checkpoints. Under the local first-order gate approximation, $\omega$ is the component that persists under example averaging, not a realized training gradient or parameter update. Shading summarizes model-level heterogeneity, not sampling uncertainty.}
\label{fig:geometry_gradient}
\end{figure*}

Table~\ref{tab:geometry-gradient-summary} provides a compact model-level account of energy--$\tau$ association and concentration; Figure~\ref{fig:geometry_gradient} additionally reports $\tau$--$\omega$ association and $\sum_i\omega_i/\sum_i\tau_i$. Proposition~\ref{prop:gate_averaging} interprets the latter as the mean-aligned component on the finite scored-row distribution; it is a local supervised diagnostic rather than a population update direction.

\begin{table}[t]
\centering
\small
\setlength{\tabcolsep}{3.4pt}
\caption{Association between activation energy and supervised-gradient salience.}
\label{tab:geometry-gradient-summary}
\resizebox{\columnwidth}{!}{%
\begin{tabular}{@{}lrrr@{}}
\toprule
Model & $\rho(E,\tau)$ & Top-decile overlap (\%) & Top-10-concept $\tau$ mass (\%) \\
\midrule
Pythia-70M & +0.87 & 76.9 & 41.6 \\
Pythia-160M & +0.89 & 74.4 & 32.4 \\
Pythia-410M & +0.85 & 69.2 & 30.2 \\
OLMo-1B & +0.90 & 73.3 & 41.2 \\
EuroBERT-210M & +0.92 & 82.1 & 25.7 \\
EuroBERT-610M & +0.87 & 75.9 & 28.0 \\
ModernCamemBERT & +0.96 & 85.9 & 42.5 \\
\bottomrule
\end{tabular}%
}
\vspace{2pt}
\begin{minipage}{0.98\linewidth}
\footnotesize\textit{Notes.} Entries are medians over method$\times$checkpoint$\times$layer concept cells. $E$ is activation energy and $\tau$ is the saved supervised-gradient second moment; $\tau$ is not identified with Fisher salience. Top-decile overlap is the intersection of the two top-decile sets divided by their common set size; the final column is the mass in the ten concepts with largest $\tau$. These descriptive associations do not track concept identities across checkpoints and do not imply a coherent parameter-update mechanism.
\end{minipage}
\end{table}

The enrichment result establishes concentration beyond the uniform size expectation. Figure~\ref{fig:same_budget_mass} adds the stricter equal-cardinality energy-head and exact additive-oracle controls. Across retained SVD downstream cells, median capture is $0.991$ of the exact top-$k$ $\tau$ oracle and $0.992$ of the $\omega$ oracle, although the energy-ranked prefix is usually at least as strong. Table~\ref{tab:same-budget-oracle} extends the comparison across all decompositions and shows substantial but model-dependent oracle capture.

\begin{figure*}[t]
\centering
\includegraphics[width=.96\textwidth]{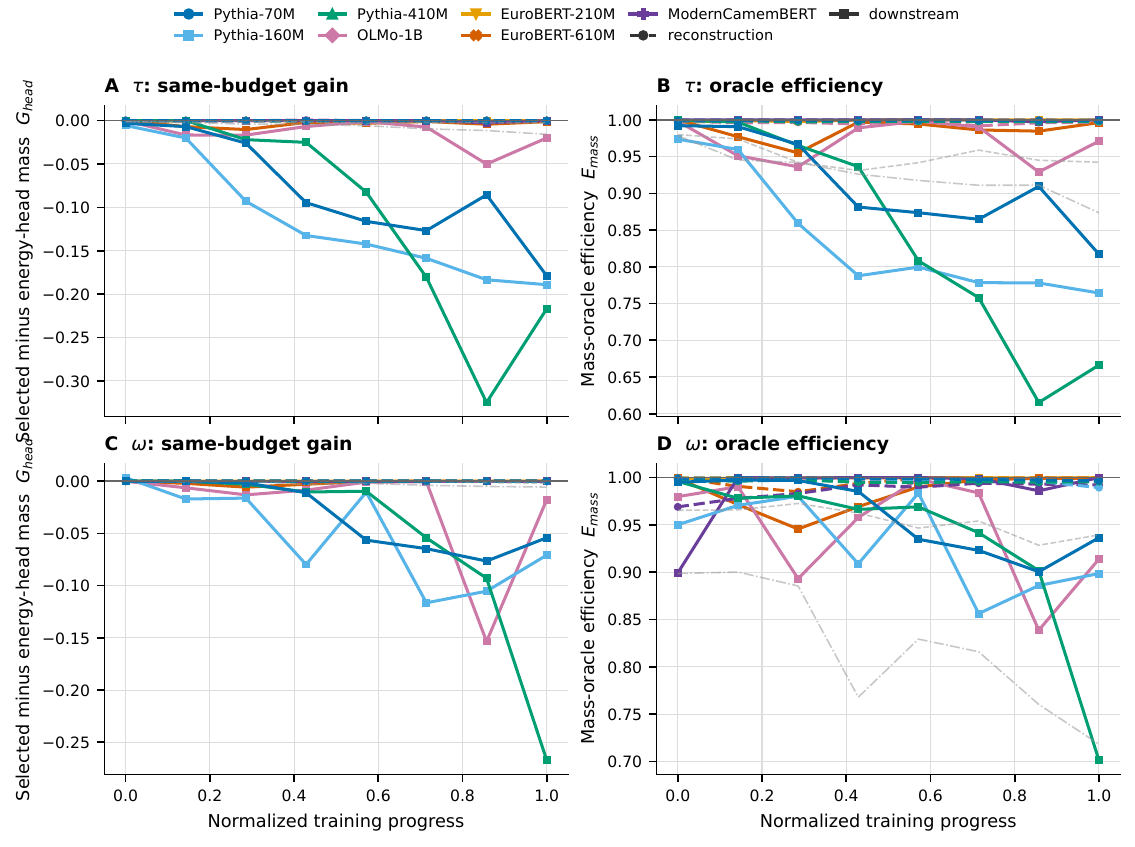}
\caption{\textbf{Same-budget gate-score control.} Colored curves are SVD layer-median trajectories for both reconstruction (dashed circles) and downstream (solid squares) selection in each model; faint gray curves are non-SVD method-median robustness summaries. Left panels show selected-minus-energy-head mass and right panels show efficiency relative to the exact additive top-$k$ score oracle, for $\tau$ (top) and $\omega$ (bottom). The progress axis is scheduled-checkpoint rank; empty selections and zero-total masses are excluded.}
\label{fig:same_budget_mass}
\end{figure*}

\begin{table*}[t]
\centering
\scriptsize
\setlength{\tabcolsep}{3.8pt}
\caption{Downstream selections versus equal-cardinality activation-energy heads and exact additive score oracles.}
\label{tab:same-budget-oracle}
\resizebox{\textwidth}{!}{%
\begin{tabular}{@{}lrrrrrrr@{}}
\toprule
Model & Valid & $G_{\rm head}(\tau)$ & $G_{\rm head}(\omega)$ & $E_{\rm mass}(\tau)$ & $E_{\rm mass}(\omega)$ & $G_{\rm head}(\tau)>0$ (\%) & $G_{\rm head}(\omega)>0$ (\%) \\
\midrule
Pythia-70M & 192 & $-0.051$ & $-0.026$ & 0.909 & 0.938 & 0.5 & 8.3 \\
Pythia-160M & 189 & $-0.087$ & $-0.054$ & 0.767 & 0.781 & 3.2 & 9.5 \\
Pythia-410M & 168 & $-0.077$ & $-0.089$ & 0.725 & 0.724 & 6.0 & 14.9 \\
OLMo-1B & 151 & $-0.030$ & $-0.041$ & 0.742 & 0.410 & 8.6 & 15.2 \\
EuroBERT-210M & 167 & $-0.002$ & $-2.22\!\times\!10^{-16}$ & 0.991 & 0.983 & 19.8 & 40.1 \\
EuroBERT-610M & 172 & $-9.28\!\times\!10^{-4}$ & $-5.06\!\times\!10^{-5}$ & 0.982 & 0.944 & 20.9 & 34.3 \\
ModernCamemBERT & 132 & $-0.001$ & $-5.03\!\times\!10^{-17}$ & 0.991 & 0.944 & 18.2 & 37.9 \\
\bottomrule
\end{tabular}%
}
\vspace{2pt}
\begin{minipage}{0.98\linewidth}
\footnotesize\textit{Notes.} Rows independently median-aggregate valid downstream cells over all four decompositions, eight checkpoints, and six tracked layers within each model. $G_{\rm head}(q)$ is selected mass minus the equal-size activation-energy-head mass; $E_{\rm mass}(q)$ is selected mass divided by the exact additive top-$k$ mass. Empty hard selections and zero-total score cells are excluded where the statistic is undefined. The activation-energy head is a strong baseline, while the oracle-efficiency columns quantify how much supervised score mass the behaviorally selected set captures.
\end{minipage}
\end{table*}

For the SVD endpoint analysis, let $e_i$ be retained-component energy, $q_i=e_i/\sum_j e_j$, and $d_E=\exp(-\sum_i q_i\log q_i)/K_{\max}$ (with $0\log 0=0$), an energy-weighted entropy-effective-number analogue of effective rank \citep{roy2007effective}. Let $T_{.90}/K_{\max}$ be the energy-ranked prefix length needed to capture $90\%$ of $\tau$. Figure~\ref{fig:svd_tail} compares the first and final common-grid checkpoints. The energy effective number decreases at five of six relative layers in both architecture aggregates. Downstream $\Delta\mathrm{KL}$ rises at all six decoder layer ranks (final/early ratios $1.40$--$4.57\times$) but falls at the first five encoder ranks ($0.15$--$0.92\times$), before rising at the final rank ($2.53\times$). This provides a depth-resolved architecture signature; the groups also differ in data, objective, and training recipe.

\begin{figure*}[t]
\centering
\includegraphics[width=.94\textwidth]{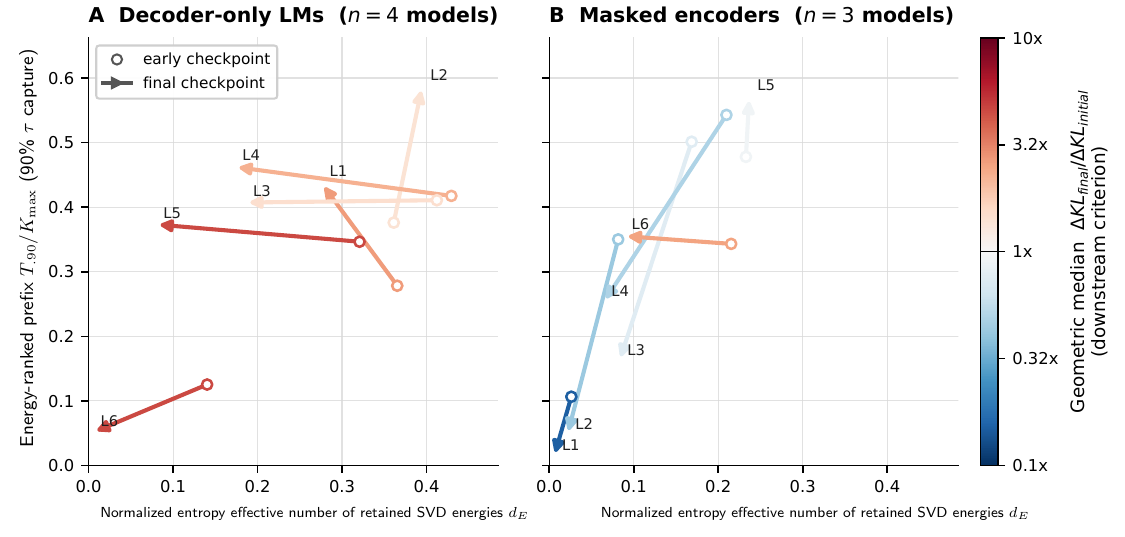}
\caption{\textbf{SVD energy concentration and supervised-gradient coverage.} Circles and arrowheads denote the earliest and final common-grid checkpoints. Coordinates are architecture-group medians at each of six relative layer ranks: horizontal position is normalized SVD energy effective number and vertical position is the energy-ranked prefix needed to capture $90\%$ of $\tau$. Color gives the geometric-median downstream $\Delta\mathrm{KL}_{\rm final}/\Delta\mathrm{KL}_{\rm early}$ ratio. The vertical coordinate is an ordering prefix, not a subspace dimension.}
\label{fig:svd_tail}
\end{figure*}

For Figure~\ref{fig:space_phase}, let $b_c=K_{\rm soft}^c/K_{\max}$, $A=d_E(t_0)-d_E(T)$, and $B=[b_{\rm down}-b_{\rm rec}]_T-[b_{\rm down}-b_{\rm rec}]_{t_0}$. Energy concentrates ($A>0$) in 36 of 42 SVD model--layer pairs, while relative downstream broadening ($B>0$) occurs in 28 of 42. All three Pythia model medians have $B>0$, compared with none of the four non-Pythia medians, a descriptive family contrast. Point color is the normalized downstream endpoint change
\begin{equation}
L=\log_{10}\!\frac{\Delta\mathrm{KL}_{\rm down}(T)/\mathrm{CE}_{0,\rm down}(T)}{\Delta\mathrm{KL}_{\rm down}(t_0)/\mathrm{CE}_{0,\rm down}(t_0)}.
\end{equation}
These summaries track allocation rather than concept identity or ambient dimension.

\begin{figure*}[t]
\centering
\includegraphics[width=.98\textwidth]{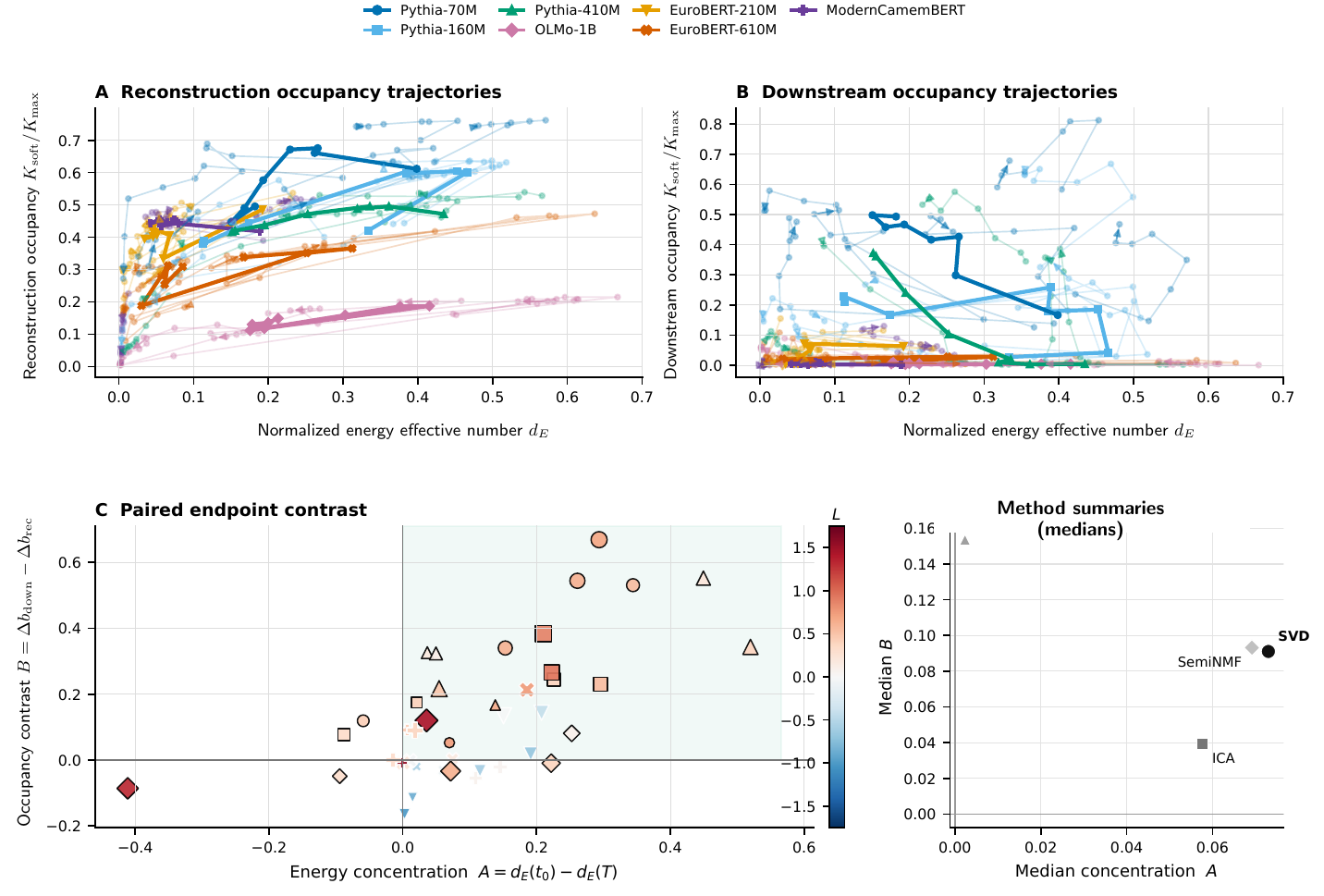}
\caption{\textbf{Endpoint energy--occupancy phase portrait.} A--B show faint SVD model--layer trajectories and thick within-model layer medians, relating normalized energy effective number to reconstruction or downstream soft occupancy. C plots the endpoint contrasts $A$ and $B$; marker size encodes relative depth, shape encodes model, and color encodes the normalized downstream $\Delta\mathrm{KL}$ change $L$ defined above. The right panel gives decomposition medians as a robustness summary. Occupancy is retained capacity, not dimensionality.}
\label{fig:space_phase}
\end{figure*}

\FloatBarrier

\subsection{Protocol and output-metric diagnostics}

All stored TwoNN estimates lie below the lower clamp of one quarter of ambient width (Figure~\ref{fig:kmax_clamp}). Thus the executed $K_{\max}$ is a consistent width-based candidate budget within each model, and mask selection---rather than variation in the raw TwoNN estimate---determines retained capacity.

\begin{figure*}[t]
\centering
\includegraphics[width=.72\textwidth]{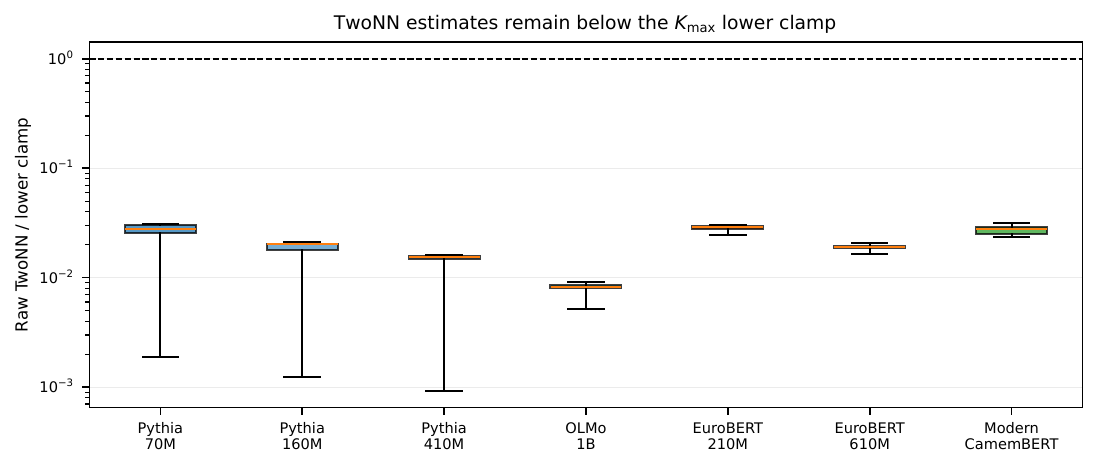}
\caption{\textbf{Candidate-budget diagnostic.} Distributions show raw TwoNN estimates divided by the implementation's lower clamp. Every value lies below one, so the clamped estimate determines every executed budget.}
\label{fig:kmax_clamp}
\end{figure*}

Figure~\ref{fig:behavioral_agreement} compares the three behavioral monitors within checkpoint strata. The high pooled KL--CE correlations reported in the main text are accompanied by informative model-level structure: checkpoint-wise correlations remain stronger in autoregressive than masked models, while CE--accuracy signs agree in $80.4\%$ of reconstruction and $89.6\%$ of downstream cells after model macro-averaging. Each correlation uses the eight scheduled checkpoints at fixed model, method, target, and layer; Panel D retains the exact ties induced by discrete accuracy.

\begin{figure*}[t]
\centering
\includegraphics[width=.98\textwidth]{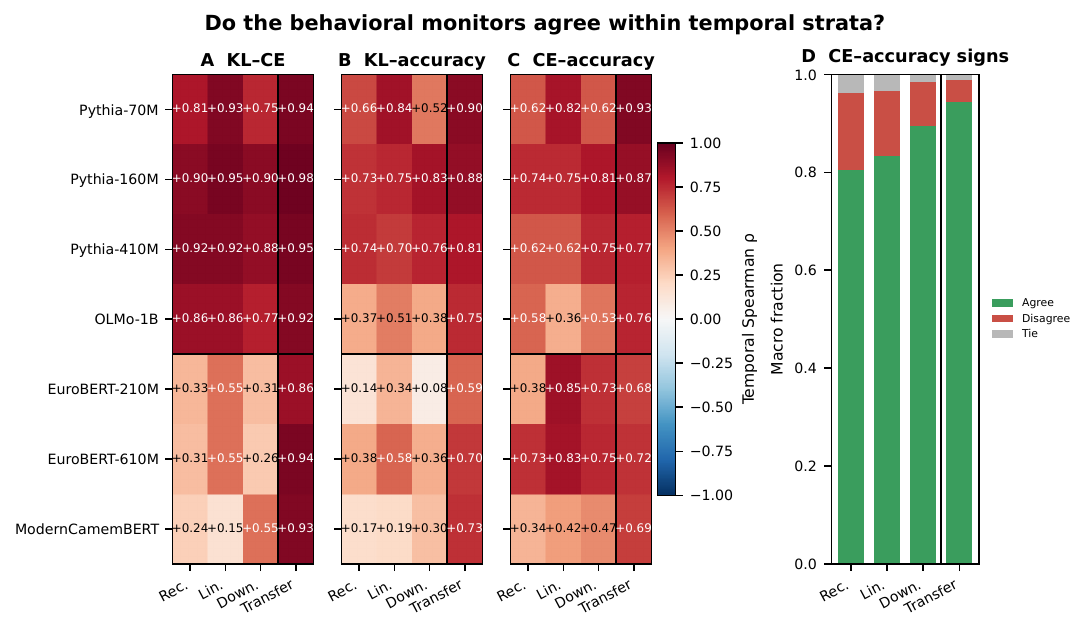}
\caption{\textbf{Behavioral-monitor agreement within temporal strata.} A--C give model-level medians of within-method/layer Spearman correlations over eight scheduled checkpoints for predictive-distribution shift, relative target-token CE damage, and target-token accuracy damage. D reports macro-averaged CE--accuracy sign agreement, disagreement, and exact ties. Transfer is visually separated, includes learned self-alignment, and is not a continuity trend. Cells are structured repeated measurements, not IID replicates.}
\label{fig:behavioral_agreement}
\end{figure*}

Figure~\ref{fig:soft_hard} connects the continuous fitted gates used for held-out intervention to the thresholded sets used for composition. Positive soft mass can be distributed among gates that all remain below $0.5$, so such cells still provide behavioral measurements even when the hard set is empty. Among the 5,376 common-grid cells, 316 hard masks are empty; model-macro rates are $0.0\%$ for reconstruction, $1.0\%$ for fixed-head selection, $12.9\%$ for downstream selection, and $9.6\%$ for transfer. The occupancy gap compares cardinalities, and the selected-side margin is the smallest retained gate's distance above the reporting threshold.

\begin{figure*}[t]
\centering
\includegraphics[width=.98\textwidth]{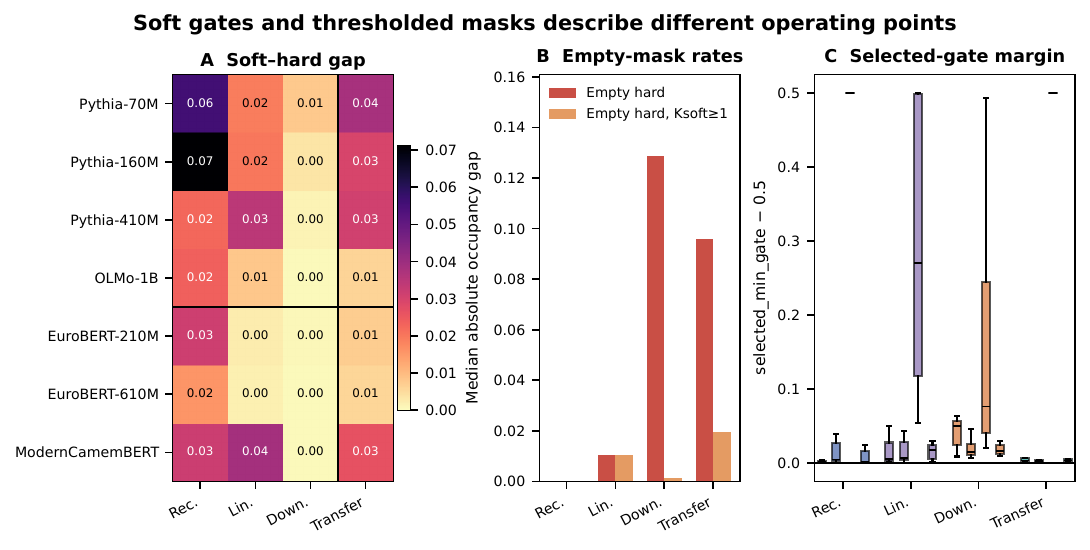}
\caption{\textbf{Soft-gate and hard-mask threshold diagnostics.} A reports the median absolute difference between normalized soft and hard cardinality. B separates all empty hard masks from the subset whose soft cardinality is at least one. C shows the smallest selected gate's margin above $0.5$ for nonempty hard masks, by target and decomposition. Together the panels explain how a continuous operating point maps to the reported hard set.}
\label{fig:soft_hard}
\end{figure*}

\FloatBarrier
\clearpage

\section{Proofs for Section~\ref{sec:theory}}
\label{app:theory_proofs}

\begin{proof}[Proof of Proposition~\ref{prop:local_downstream_approx}]
For fixed $X$, define
\[
K_X(\delta)=
\KL\!\left(
p_t(\cdot\mid a(X))
\middle\|
p_t(\cdot\mid a(X)+\delta)
\right).
\]
The first derivative of $K_X$ vanishes at $\delta=0$. Its Hessian at zero is $G_t(a(X))$, obtained by applying the chain rule to the softmax KL Hessian $S(p_t(\cdot\mid a(X)))$; terms involving second derivatives of the downstream logit map vanish because the logit-space KL gradient is zero at $\delta=0$. Set
$q_X(\delta)=\frac12\delta^\top G_t(a(X))\delta$.
Taylor's theorem gives
\begin{align*}
&\left|K_X(\Delta_z(X))-q_X(\Delta_z(X))\right|
\\
&\qquad\le
\frac{B_X(z)}{6}\|\Delta_z(X)\|_2^3.
\end{align*}
Taking the expectation proves Equation~\eqref{eq:local_downstream_bound}. For the optimizer comparison, let
$\eta_D=\eta_t(z_D)$ and $\eta_Q=\eta_t(z_Q)$. Because the shared penalty cancels, $|F_D-F_Q|=|D_t-Q_t|$. Optimality of $z_Q$ for $F_Q$ gives
\[
F_Q(z_Q)\le F_Q(z_D)\le F_D(z_D)+\eta_D.
\]
The local bound at $z_Q$ then gives
\[
F_D(z_Q)\le F_Q(z_Q)+\eta_Q
\le F_D(z_D)+\eta_D+\eta_Q,
\]
which proves Equation~\eqref{eq:local_optimizer_comparison}.
\end{proof}

\begin{proof}[Proof of Proposition~\ref{prop:gate_averaging}]
For every coordinate $i$, independence of the $B$ examples gives
\begin{align}
\E\!\left[\bar g_{B,i}^{2}\right]
&=(\E g_i)^2+\frac{\mathrm{Var}(g_i)}{B}
\nonumber\\
&=\omega_i+\frac{\kappa_i}{B}
=\left(1-\frac1B\right)\omega_i+\frac{\tau_i}{B}.
\end{align}
Summing over $i\in A$ proves Equation~\eqref{eq:gate_batch_identity}. No cross-coordinate assumption is required because the squared Euclidean norm is the sum of the coordinate squares.
\end{proof}

For completeness, let $\bar g_B=B^{-1}\sum_{b=1}^B(g_i(X_b))_{i=1}^{K_{\max}}$ be the full gate-gradient vector for a realized batch, and consider a first-order gate change $\delta$ supported on $A$. Cauchy--Schwarz gives
\begin{equation}
\max_{\substack{\operatorname{supp}(\delta)\subseteq A\\
\|\delta\|_2\le\varepsilon}}
-\langle\bar g_B,\delta\rangle
=\varepsilon\|\bar g_{B,A}\|_2.
\label{eq:gate_linear_potential}
\end{equation}
This is the largest \emph{linearized} loss decrease under the diagnostic gate constraint, not an observed training update.

\paragraph{Consequence for an SVD energy prefix.}
Order the SVD components of the fitted activation matrix by decreasing activation energy $e_{(1)}\ge\cdots\ge e_{(K_{\max})}$, and write $H_k=\{(1),\ldots,(k)\}$. These prefixes are optimal for rank-$k$ Euclidean reconstruction \citep{eckart1936approximation,jolliffe2002principal}. Writing $H_k^c$ for the complement among the $K_{\max}$ fitted coordinates, Proposition~\ref{prop:gate_averaging} gives the exact omitted batch-gate signal
\begin{equation}
\E\|\bar g_{B,H_k^c}\|_2^2
=\sum_{i>k}\left[
\left(1-\frac1B\right)\omega_{(i)}+\frac{\tau_{(i)}}B
\right].
\label{eq:svd_tail_gate_identity}
\end{equation}
For $\alpha\in(0,1]$ and $\sum_i\tau_i>0$, define the energy-ranked coverage prefix
\begin{equation}
T_\alpha=
\min\left\{k:
\frac{\sum_{i=1}^{k}\tau_{(i)}}{\sum_j\tau_j}\ge\alpha
\right\}.
\label{eq:salience_coverage_depth}
\end{equation}
If $k<T_\alpha$, more than $1-\alpha$ of the single-example squared gate signal lies outside $H_k$. At larger $B$, this conclusion depends increasingly on the stored $\omega$ mass in the tail. Equations~\eqref{eq:svd_tail_gate_identity}--\eqref{eq:salience_coverage_depth} are within-checkpoint statements: they neither align coordinates over training nor imply that the energy prefix consumes ambient dimensions.
\end{document}